\documentclass[12pt,oneside,british,english]{amsart}
\usepackage{amsbsy}
\usepackage{amstext}
\usepackage{amsthm}
\usepackage{mathptmx}
\usepackage{newtxmath}

\usepackage[T1]{fontenc}
\usepackage[latin9]{inputenc}
\usepackage{babel}
\usepackage{float}
\usepackage{mathtools}
\usepackage{esint}
\usepackage[unicode=true,pdfusetitle,
 bookmarks=true,bookmarksnumbered=false,bookmarksopen=false,
 breaklinks=false,pdfborder={0 0 1},backref=false,colorlinks=false]
 {hyperref}
\usepackage{breakurl}

\makeatletter

\providecommand{\tabularnewline}{\\}
\floatstyle{ruled}
\newfloat{algorithm}{tbp}{loa}
\providecommand{\algorithmname}{Algorithm}
\floatname{algorithm}{\protect\algorithmname}

\numberwithin{equation}{section}
\numberwithin{figure}{section}
\theoremstyle{plain}
\newtheorem{thm}{\protect\theoremname}
  \theoremstyle{remark}
  \newtheorem{rem}[thm]{\protect\remarkname}
  \theoremstyle{definition}
  \newtheorem{example}[thm]{\protect\examplename}
  \theoremstyle{plain}
  \newtheorem{cor}[thm]{\protect\corollaryname}
  \theoremstyle{definition}
  \newtheorem{defn}[thm]{\protect\definitionname}
  \theoremstyle{plain}
  \newtheorem{lem}[thm]{\protect\lemmaname}
  \theoremstyle{plain}
  \newtheorem{prop}[thm]{\protect\propositionname}
  \theoremstyle{remark}
  \newtheorem{acknowledgement}[thm]{\protect\acknowledgementname}

\usepackage{algorithm,algpseudocode} 
\usepackage{todonotes}
\usepackage{babel}
\usepackage[a4paper, margin=3cm]{geometry}
\usepackage{newtxtext,newtxmath}
\usepackage{dsfont} 
\usepackage{microtype}
\usepackage{siunitx}
\usepackage{booktabs}
\usepackage{cleveref}
\usepackage{fixltx2e}
\usepackage{breakurl}
\usepackage{tikz, pgf, pgfplots}
\usepackage{algorithm,algpseudocode}

\makeatother

  \addto\captionsbritish{\renewcommand{\acknowledgementname}{Acknowledgement}}
  \addto\captionsbritish{\renewcommand{\corollaryname}{Corollary}}
  \addto\captionsbritish{\renewcommand{\definitionname}{Definition}}
  \addto\captionsbritish{\renewcommand{\examplename}{Example}}
  \addto\captionsbritish{\renewcommand{\lemmaname}{Lemma}}
  \addto\captionsbritish{\renewcommand{\propositionname}{Proposition}}
  \addto\captionsbritish{\renewcommand{\remarkname}{Remark}}
  \addto\captionsbritish{\renewcommand{\theoremname}{Theorem}}
  \addto\captionsenglish{\renewcommand{\acknowledgementname}{Acknowledgement}}
  \addto\captionsenglish{\renewcommand{\corollaryname}{Corollary}}
  \addto\captionsenglish{\renewcommand{\definitionname}{Definition}}
  \addto\captionsenglish{\renewcommand{\examplename}{Example}}
  \addto\captionsenglish{\renewcommand{\lemmaname}{Lemma}}
  \addto\captionsenglish{\renewcommand{\propositionname}{Proposition}}
  \addto\captionsenglish{\renewcommand{\remarkname}{Remark}}
  \addto\captionsenglish{\renewcommand{\theoremname}{Theorem}}
  \providecommand{\acknowledgementname}{Acknowledgement}
  \providecommand{\corollaryname}{Corollary}
  \providecommand{\definitionname}{Definition}
  \providecommand{\examplename}{Example}
  \providecommand{\lemmaname}{Lemma}
  \providecommand{\propositionname}{Proposition}
  \providecommand{\remarkname}{Remark}
\addto\captionsbritish{\renewcommand{\algorithmname}{Algorithm}}
\providecommand{\theoremname}{Theorem}

\begin{document}

\title{sketching the order of events }

\author{Terry Lyons}

\email{tlyons@maths.ox.ac.uk}

\author{Harald Oberhauser}

\email{oberhauser@maths.ox.ac.uk}

\address{Mathematical Institute, University of Oxford}
\begin{abstract}
We introduce features for massive data streams. These stream features
can be thought of as ``ordered moments'' and generalize stream sketches
from ``moments of order one'' to ``ordered moments of arbitrary
order''. In analogy to classic moments, they have theoretical guarantees
such as universality that are important for learning algorithms.  
\end{abstract}

\maketitle
\global\long\def\stream{\operatorname{Streams}}

\global\long\def\tensalg{\operatorname{T}}

\global\long\def\alph{A^{\star}}

\global\long\def\letters{A}

\global\long\def\free{\mathbb{R}A}

\global\long\def\alphsmall{B}

\global\long\def\sig{\operatorname{\psi}}

\global\long\def\magnus{\varphi}

\global\long\def\magnusgappy{\varphi_{g}}

\global\long\def\magnuspol{\operatorname{p_{1}}}

\global\long\def\sigsketch{\operatorname{S}_{\mathcal{H}}^{\text{ }}}

\global\long\def\bias{\operatorname{bias}}

\global\long\def\shuffle{\operatorname{m}_{p}}

\global\long\def\coshuffle{\operatorname{\Delta}_{\text{S}}}

\global\long\def\concat{\operatorname{m}_{c}}

\global\long\def\coconcat{\operatorname{\Delta_{C}}}

\global\long\def\infil{\operatorname{m_{I}}}

\global\long\def\coinfil{\operatorname{\Delta_{I}}}

\global\long\def\ta{\mathcal{F}^{\prime}}

\global\long\def\tsmalla{\mathbb{R}\left\langle \alphsmall\right\rangle }

\global\long\def\tsmallb{\operatorname{T}\left(\free_{\alphsmall}\right)}

\global\long\def\freevs{F\left(\letters\right)}

\global\long\def\tb{\mathcal{F}}

\global\long\def\tbsmall{\mathbb{R}\left\langle \left\langle \alphsmall\right\rangle \right\rangle }

\global\long\def\letter{\ell}

\global\long\def\simplexstrict{\Delta_{<}}

\global\long\def\simplex{\Delta}

\global\long\def\lattice{\mathcal{P}_{\text{lattice}}}

\global\long\def\lattice{\mathcal{P}_{\text{lattice}}^{+}}

\renewcommand{\algorithmicensure}{\textbf{Output:}}
\renewcommand{\algorithmicrequire}{\textbf{Process:}}
\algnewcommand\algorithmicinit{\textbf{Initialize:}}
\algnewcommand\init{\item[\algorithmicinit]}

\section{Introduction\label{sec:Introduction}}

\subsection{The streaming problem}

A stream $\sigma=\left(\sigma_{i}\right)_{i=1}^{L}$ is a sequence
of events. An event $\sigma_{i}$ is a tuple $\sigma_{i}=\left(\lambda_{i},a_{i}\right)\in\mathbb{R}\times\letters$
where $\letters$ denotes a finite but typically very large set. Our
task is to compute a summary of the stream 
\[
\Phi\left(\sigma_{1},\ldots,\sigma_{L}\right)
\]
and to update it on arrival of a new event. This summary $\Phi\left(\sigma\right)$
should be rich enough to efficiently \emph{describe the effects of
the stream $\sigma=\left(\sigma_{1},\ldots,\sigma_{L}\right)$, }that
is allow to make inference about functions $f\left(\sigma\right)$
of the stream. We refer to $\Phi$ as \emph{feature map }and in this
article we focus on so-called \emph{cash register streams}, that is
the space of events $\mathcal{E}=\mathbb{R}_{>0}\times\letters$ has
only positive increments. We call $\sigma_{i}=\left(\lambda_{i},a_{i}\right)\in\mathcal{E}$
the \emph{event }with \emph{counter increase $\lambda_{i}\geq0$ }in
the \emph{letter }$a$ and we call $\mathcal{S}=\bigcup_{L\geq1}\mathcal{E}^{L}$
the set of \emph{cash register streams}; see \cite{muthukrishnan2005data}
for more background on data streaming.

\subsection{Examples\label{sub:Examples}}

Streams taking values in large sets arise in many applications: parsing
a text word-by-word or letter-by-letter ($\lambda_{i}\equiv1$ with
$\left|\letters\right|\simeq2^{7}$ if ASCII characters or $\left|\letters\right|\simeq10^{5}$
if English dictionary words are parsed); recording network traffic
in a router ($\lambda_{i}>0$ is the data volume and $\left|\letters\right|\simeq10^{38}$
IP adresses); in the order book of a stock exchange ($\lambda_{i}>0$
denoting trading volume, $\left|\letters\right|\simeq10^{4}$ traded
assets); etc. Effects of streams are functions $f\left(\sigma_{1},\ldots,\sigma_{L}\right)$:
the function that categorizes texts into ``drama'', ``comedy'',
``news'', ``gossip''; the function that decides if a network traffic
stream contains abnormal patterns; the functions that detects trading
patterns in stocks; etc. In all these examples, the order in which
elements of the stream are received carries relevant information.

\subsection{Features\label{sub:Features}}

We construct a map $\Phi$ from the space of streams into a linear
space 
\[
\sigma=\left(\lambda_{i},a_{i}\right)\mapsto\Phi\left(\sigma\right)
\]
such that
\begin{enumerate}
\item \label{enu:computed}\textbf{(Efficient algorithms) }$\Phi\left(\sigma\right)$
can be well approximated 

\begin{enumerate}
\item in logarithmic space complexity in $\left|\letters\right|$,
\item in ``streaming fashion'': with a single pass over the stream $\sigma=\left(\sigma_{i}\right)$.
\end{enumerate}
\item \label{enu:linearizes}\textbf{(Universal features) }$\Phi$ linearizes
non-linear functionals $f$ of streams, i.e. 
\[
f\left(\sigma\right)\simeq\left\langle \ell,\Phi\left(\sigma\right)\right\rangle 
\]
where $\ell$ is a linear functional of $\Phi\left(\sigma\right)$
and above holds uniformly over streams $\sigma$. This is known as
``universality'' in the machine learning literature and justifies
the use of standard learning algorithms such as linear classifiers.
\item \label{enu:features algebra}\textbf{(Pattern queries) }The coordinates
of $\Phi\left(\sigma\right)$ are indexed by words build from the
alphabet $\letters$ and have natural interpretation as counting patterns
in $\sigma$, e.g.~$\Phi_{a}\left(\sigma\right)=\sum_{i}\lambda_{i}1_{a_{i}=a}$
, $\Phi_{ab}\left(\sigma\right)=\sum_{i<j}\lambda_{i}\lambda_{j}1_{a_{i}=a,a_{i}=b}$.
Operations on streams become algebraic operations in feature space,
e.g.~stream concatenation amounts to a (non-commutative!) multiplication
in feature space.
\item \textbf{\label{enu:(robust-under-noise}(Scaling limits) }the feature
map allows to understand the scaling limit, that is when the number
of events in the stream becomes very large, $L\rightarrow\infty$.
Additionally, it is robust under noisy observations.
\end{enumerate}
Point (\ref{enu:computed}) is a central theme in the streaming community
with spectacular progress in recent years \cite{alon1996space,charikar2002finding,flajolet1985probabilistic,muthukrishnan2005data};
Point (\ref{enu:linearizes}) is a standard requirement for guarantees
of most machine learning algorithms (``universality of features'');
Point (\ref{enu:features algebra}) and Point (\ref{enu:(robust-under-noise})
are a central theme in stochastic analysis \cite{lyons2014rough,friz2014course}.
Providing features that address all four points is the goal of this
article.

\begin{rem}
One can identify $\sigma=\left(\sigma_{i}\right)_{i=1}^{L}$ as an
element of $\mathbb{R}^{\left|\letters\right|^{L}}$ and apply standard
features for vector-valued data. This approach becomes computationally
infeasible for large $L$. Further, methodological problems arise
since streams of different lengths are mapped to different feature
spaces, etc.
\end{rem}

\begin{rem}
In analogy to the count-min sketch \cite{charikar2002finding,charikar2004finding},
our feature sketch can be modified to deal with with $\ell_{2}$-error
bounds (instead of $\ell_{1}$) to deal with turnstile instead of
cash-register streams etc.; the needed modifications are analogous
to the classic case, see Remark \ref{rem: tursntile}.
\end{rem}

\subsection{Sketching}

Already for simple features $\Phi$ it is not easy to address Point
(\ref{enu:computed}) and often it can be shown that the computational
problem is NP-hard in space. A very successful approach to reduce
the computational complexity are so-called \emph{sketches,} that is
small data structures that rely on randomized algorithms \cite{alon1996space,charikar2002finding,flajolet1985probabilistic,muthukrishnan2005data}.
These algorithms compute for given $\epsilon,\delta>0$ a random variable
$\hat{\Phi}\left(\sigma\right)=\left(\hat{\Phi}_{i}\left(\sigma\right)\right)$,
such that the relative error is small in probability
\[
\mathbb{P}\left(\frac{\left|\Phi_{i}\left(\sigma\right)-\hat{\Phi}_{i}\left(\sigma\right)\right|}{\left\Vert \Phi\left(\sigma\right)\right\Vert }<\epsilon\right)>1-\delta\text{ for every coordinate }\Phi_{i}\left(\sigma\right)\text{ of }\Phi\left(\sigma\right)=\left(\Phi_{i}\left(\sigma\right)\right).
\]
An important case is when the feature map are letter frequencies,
that is $\Phi\left(\sigma\right)\in\mathbb{R}^{\left|\letters\right|}$
with each coordinate being the frequency of an element\footnote{That is, if $\sigma=\left(\lambda_{i},a_{i}\right)_{i=1}^{L}$ the
coordinate for $a\in\letters$ equals $\sum_{i:a_{i}=a}\lambda_{i}$.} in $\letters$ and $\left\Vert \Phi\left(\sigma\right)\right\Vert $
denotes the $1$-norm. Cormode and Muthukrishnan \cite{cormode2005improved}
show that in this case, $\hat{\Phi}\left(\sigma\right)$ can be constructed
by sampling a random matrix $L\in\mathbb{R}^{m\times\left|\letters\right|}$
with $m\ll\left|\letters\right|$ and storing only the $m$-dimensional,
random vector 
\begin{equation}
L\Phi\left(\sigma\right).\label{eq:rand proj}
\end{equation}
In contrast to compressed sensing, the entries of $L$ are not i.i.d.~but
have a rich structure. This allows to construct the estimator $\hat{\Phi}\left(\sigma\right)$
from the low-dimensional, random vector (\ref{eq:rand proj}) without
solving a constrained minimization problem. Moreover this sketch is
linear in the sense that, 
\[
\hat{\Phi}\left(\alpha\sigma,\beta\tau\right)=\alpha\hat{\Phi}\left(\sigma\right)+\beta\hat{\Phi}\left(\gamma\right)\text{ for }\alpha,\beta\in\mathbb{R}
\]
where multiplication with scalars of streams is defined as $\alpha\sigma=\left(\alpha\lambda_{i},a_{i}\right)$
and $\left(\alpha\sigma,\beta\tau\right)$ is simply the concatenation
of the streams $\alpha\sigma$ and $\beta\tau$. A drawback of the
such choices for $\Phi$ (such as letter frequencies, number of distinct
letters, moments of frequencies, etc.) is that they are order agnostic,
\[
\Phi\left(\sigma\right)=\Phi\left(\sigma^{\pi}\right)
\]
where $\sigma^{\pi}=\left(\sigma_{\pi\left(1\right)},\ldots,\sigma_{\pi\left(L\right)}\right)$
and $\pi$ is a permutation of $\left\{ 1,\ldots,L\right\} $. While
such order agnostic features are useful and widely used in practice,
they only very partially fulfill Points (\ref{enu:linearizes}), (\ref{enu:features algebra}),
(\ref{enu:(robust-under-noise}) and for many applications the order
information carries important information, see the above examples
\ref{sub:Examples}.

\subsection{Related work and contributions}

Finding efficient summaries of patterns in sequences (also called
substrings, motifs, etc.) is a classic theme in computer science,
data mining and machine learning \cite{leskovec2014mining}: 
\begin{enumerate}
\item A standard machine learning approach are string kernels. These capture
the order of events and have already been combined with sketches \cite{Shi2009HashK}.
However, they only apply to constant counter-increases $\lambda_{i}\equiv1$,
and the behavior and universality as $L\rightarrow\infty$ (Point
\ref{enu:linearizes} and Point \ref{enu:(robust-under-noise}) is
not discussed. 
\item The engineering community developed algorithms such as SPADE, APriori,
Freespan, \cite{han2000freespan,zaki2001spade}, etc.~to find patterns.
Usually, properties such as universality that are relevant to machine
learning tasks are not discussed (Point (\ref{enu:linearizes}),(\ref{enu:features algebra}),(\ref{enu:(robust-under-noise})
above). 
\item Describing a sequence as a formal power series in non-commutative
variables is a well-known technique in many areas of mathematics \cite{chen-58,MR2630037,lyons2014rough}.
For small alphabets, this was used for for various learning/statistical
tasks \cite{papavasiliou2011parameter,yang2017leveraging,graham:handwritting,2017arXiv170801206K};
further, the kernelization developed in \cite{2016arXiv160108169K}
covers large alphabets and string/ANOVA/time warping kernels arise
as a special cases; however, the latter restricts to kernelized learning
algorithms and gives no pattern queries. 
\item Sketches have been used for learning and optimization tasks \cite{fowkes2016subsequence,pilanci2015randomized,yang2015randomized},
see for example \cite{indyk2000identifying} for detection of trends
in time series, but in a stream/sequential context these results are
usually discussed without theoretical guarantees such as universality
or scaling limits, Points (\ref{enu:linearizes}) and (\ref{enu:(robust-under-noise}).
\end{enumerate}
To sum up: some classic constructions (such as string kernel features,
frequency sketches, etc.) arise as special cases or are realted to
our approach. However, the central theme of our approach is to focus
on the functions of stream and how sketches of features allow to linearize
such functions. The key to to this is the algebraic-analytic background
that allows to prove properties that address Points (\ref{enu:linearizes}),
(\ref{enu:features algebra}), (\ref{enu:(robust-under-noise}). More
precisely, is capture in the interplay of two algebras
\begin{itemize}
\item a graded, non-commutative algebra as the feature space (to capture
order information),
\item a graded, commutative algebra of linear functionals of features (that
is dense among function on streams),
\end{itemize}
which can be elegantly formulated as an Hopf algebra. Another perspective
that we develop in a streaming context is to identify a stream as
a lattice path in the free vector space spanned by the alphabet $\letters$.
This elementary observation allows us to apply insights from stochastic
analysis and rough path theory to the study of streams, e.g.~it gives
a useful topology on $\mathcal{S}$ and clarifies the behaviour as
then number of events goes to infinity, etc. In turn, developing sketching
ideas from this perspective allows for efficient computation, thus
also addressing Point (\ref{enu:computed}) which ultimately allows
to learn nonlinear functions of massive data streams.
\begin{rem}
We are motivated by the count-min sketch. The underlying principle
(streams as paths that are injected into the algebra of non-commutative
polynomials) is not restricted to the count-min sketch and it is an
interesting question how it can be applied to other classic sketching
algorithms such as \cite{alon1996space,datar2002maintaining}, or
other approaches to summarize massive streams such as sampling \cite{2017arXiv170302693D}.
\end{rem}

\subsection{Applications and experiments}

In Section \ref{sec:Algorithms,-Applications-and} we discuss applications.
These include pattern queries and building the list of patterns of
heavy hitters with a single parse over the stream; $M=1$ recovers
the usual heavy hitter sketch. Many effects of streams are given by
functions $f\left(\sigma\right)$ that are well approximated by considering
only the substream $\sigma^{\mathcal{H}}$ consisting of heavy hitters,
$f\left(\sigma\right)\approx f\left(\sigma^{\mathcal{H}}\right)$.
By combining universality of $\Phi$ and our sketching result this
allows to approximate 
\[
f\left(\sigma\right)\approx f\left(\sigma^{\mathcal{H}}\right)\approx\left\langle \ell,\hat{\Phi}\left(\sigma^{\mathcal{H}}\right)\right\rangle .
\]
For example, $f$ could assign a label to streams (normal/abnormal
stream) and one can train a standard learning algorithm to find the
linear functional $\ell$.

In Section \ref{sec:Experiments} we give numerical examples. Our
algorithms are easy to parallelize and thus can make use of multi-threading
on several CPUs or GPUs to parse high-volume streams. We implemented
our algorithms in C++ and ran the following experiments on synthetic
data:
\begin{enumerate}
\item \textbf{(Speed and ground truth). }We evaluated our algorithms on
a deterministic stream. The dimensions of the stream was chosen such
one can still calculate the ground truth for ordered moments up to
$M=3$. We then compared it to the error introduced by sketching and
the speed. 
\item \textbf{(Classifying Markov chains).} We sample streams from two Markov
chains in a state space $\letters$ consisting of $\left|\letters\right|=10^{5}$
letters. Even in such a simple toy model of high-dimensional streams,
order-agnostic features can become quickly useless but sketches of
patterns allow to efficiently train classifiers.
\end{enumerate}

\section{from local events to global summaries\label{sec:Hopf-algebras-as}}

Fix a finite set $\letters$ of \emph{letters }and denote with 
\begin{eqnarray*}
\mathcal{S} & = & \bigcup_{L\geq1}\mathcal{E}^{L}=\bigcup_{L\geq1}\left\{ \sigma=\left(\sigma_{i}\right)_{i=1}^{L}:\sigma_{i}\in\mathcal{E}\right\} 
\end{eqnarray*}
the set of turnstile streams consisting of an arbitrary number of
elements, so-called events, in\emph{ $\mathcal{E}=\mathbb{R}_{>0}\times\letters$}.
Further, denote with 
\[
\alph=\left\{ a_{1}\cdots a_{M}:a_{1},\ldots,a_{M}\in\letters,\,M\geq0\right\} 
\]
the set of \emph{words}; we denote the empty word with $1$. Let 
\[
\tb:=\mathbb{R}\left\langle \left\langle \letters\right\rangle \right\rangle :=\left\{ \sum_{w\in\alph}c_{w}w:c_{w}\in\mathbb{R}\right\} 
\]
be the linear space of formal power series in $\left|\letters\right|$
non-commuting variables. Below we show how to construct feature maps
\[
\Phi:\mathcal{S}\rightarrow\tb,\quad\sigma\mapsto\Phi\left(\sigma\right)=\sum_{w}\Phi_{w}\left(\sigma\right)w
\]
by ``stitching events together in feature space''.

\subsection{From local events to global descriptions}

Fix a so-called \textbf{event map} 
\[
p:\mathcal{E}\rightarrow\tb
\]
We evaluate $p$ along the stream $\sigma_{1},\sigma_{2},\ldots$
and update our features by (non-commutative!) multiplication in $\tb$
\begin{eqnarray}
\Phi\left(\emptyset\right) & = & 1,\nonumber \\
\Phi\left(\sigma_{1},\ldots,\sigma_{L+1}\right) & = & \Phi\left(\sigma_{1},\ldots,\sigma_{L}\right)p\left(\sigma_{L+1}\right)\label{eq:mult p}
\end{eqnarray}
The non-commutativity of the multiplication in the feature space $\tb$
captures the order information. A priori there are many possible choices
for the event map $p$, each choice results in different coefficients
$\left(\Phi_{w}\left(\sigma\right)\right)_{w\in\alph}$ of $\Phi\left(\sigma\right)=\sum\Phi_{w}\left(\sigma\right)w\in\tb$.
More importantly, different choices of $p$ give rise to different
algebraic structures on the dual of $\tb$.
\begin{example}
\label{exa:inf}Consider the feature map 
\begin{equation}
p\left(\lambda,a\right)=1+\lambda a.\label{eq:1plus}
\end{equation}
With $A=\left\{ a,b\right\} $ and $\sigma=\left(\left(1,a\right),\left(1.5,b\right),\left(1,b\right),\left(2,a\right)\right)$.
Then $\alph=\left\{ 1,a,b,a^{2},ab,ba,b^{2},a^{3},\ldots\right\} $
and 
\begin{eqnarray}
\Phi\left(\sigma\right) & = & \left(1+a\right)\left(1+1.5b\right)\left(1+b\right)\left(1+2a\right)\nonumber \\
 & = & 1+3a+2.5b+2a^{2}+2.5ab+5.5ba+1.5b^{2}.\label{eq:aa}
\end{eqnarray}
Note that each coordinate $\Phi_{w}\left(\sigma\right)$ counts how
often $w$ appears as subword in $\sigma$ relative to the counter-increases
$\lambda_{i}$; the map (\ref{eq:1plus}) was introduced in \cite{2016arXiv160108169K}
in the context of sequentializing kernels; e.g.~for a stream with
constant counter-increase $\lambda_{i}\equiv1$, $\Phi\left(\sigma\right)$
recovers the classic vanilla string kernel features \cite{lodhi2002text}
and similarly one recovers ANOVA features etc., see \cite{2016arXiv160108169K}. 
\end{example}

\begin{example}
\label{exa:sigs}Let 
\[
p\left(\lambda,a\right)=1+\lambda a+\frac{\lambda^{2}a^{2}}{2!}+\cdots.
\]
This leads to the feature map 
\begin{eqnarray}
\Phi\left(\sigma\right) & = & \left(1+a+\frac{a^{2}}{2!}+\cdots\right)\left(1+1.5b+\frac{\left(1.5b\right)^{2}}{2!}+\cdots\right)\left(1+b+\frac{b^{2}}{2!}+\cdots\right)\left(1+2a+\frac{\left(2a\right)^{2}}{2!}+\cdots\right)\nonumber \\
 & = & 1+3a+2.5b+\left(2+\frac{1}{2!}+\frac{2^{2}}{2!}\right)a^{2}+\left(1.5+1\right)ab+\cdots.\label{eq:bb}
\end{eqnarray}
The coefficients are somewhat less intuitive than in above Example
\ref{exa:inf} but lead to a classic algebraic structure (the shuffle
algebra, Appendix \ref{sec:Hopf}). Moreover, as $L\rightarrow\infty$
this choice of event map still makes sense for turnstile streams,
whereas Example \ref{exa:inf} leads to problems in the limit $L\rightarrow\infty$
when turnstile streams are considered and rough paths appear, see
the discussion \cite{2016arXiv160108169K} for further details in
a (kernel) learning context.\end{example}
\begin{rem}
The algebraic structure of Example \ref{exa:sigs} is the classic
shuffle Hopf algebra, whereas the (Hopf) algebra arising from Example
\ref{exa:inf} was at least new to us, see Theorem \ref{thm:sig hopf algebra}
and Theorem \ref{thm:scaling} in the Appendix \ref{sec:Hopf}. 
\end{rem}

\begin{rem}
An interesting question is how to construct a commutative product
on the dual of $\ta$ for a given event map $p$ and vice versa. Note
that the non-commutative product (formal power series multiplication)
in $\tb$ is the same for all choices of $p$.  
\end{rem}

\subsection{(Pseudo-)Norms}

For $f=\sum_{w}c_{w}w\in\mathcal{F}$ define 
\[
\left\Vert f\right\Vert _{1}=\sum_{w}\left|c_{w}\right|
\]
(with the convention $\left\Vert f\right\Vert _{1}=\infty$ if the
sum does not converge). For a word $w=a_{1}\cdots a_{M}\in\alph$
define its length $\left|w\right|=M$ as the number of letters in
$w$. For $M\geq1$ define
\[
\left\Vert f\right\Vert _{1,M}=\sum_{w:\left|w\right|\leq M}\left|c_{w}\right|\text{ and }\left\Vert f\right\Vert _{1,\left(M\right)}=\sum_{w:\left|w\right|=M}\left|c_{w}\right|.
\]
None of these are norms on $\mathcal{F}$ (but on appropriate sub-
or quotient-spaces). However they appear naturally in our calculations.

\subsection{Some algebra and feature universality}

Denote with 
\[
\ta:=\mathbb{R}\left\langle \letters\right\rangle =\left\{ \sum_{w\in\alph}c_{w}w:c_{w}\in\mathbb{R}\text{ and }c_{w}=0\text{ for infinitely many }w\right\} 
\]
the subset of $\tb$ that consists of finite sums. We identify elements
of $\ta$ as linear functionals acting on $\tb$ via the pairing 
\[
\left\langle \cdot,\cdot\right\rangle :\mathcal{F}^{\prime}\times\mathcal{F}\rightarrow\mathbb{R},\,\,\,\left\langle \ell,P\right\rangle :=\sum_{w}c_{w}d_{w}\text{ where }\ell=\sum_{w}d_{w}w,\,P=\sum_{w}d_{w}w.
\]
A far reaching result is that $\ta$ can be equipped with a commutative
multiplication, that is for $\ell_{1},\ell_{2}\in\ta$ there exists
a $\ell\in\ta$ that is given by multiplication of $\ell_{1},\ell_{2}$
such that 
\[
\left\langle \ell_{1},\Phi\left(\sigma\right)\right\rangle \left\langle \ell_{2},\Phi\left(\sigma\right)\right\rangle =\left\langle \ell,\Phi\left(\sigma\right)\right\rangle .
\]
Thus 
\[
\left\{ \sigma\rightarrow\left\langle \ell,\Phi\left(\sigma\right)\right\rangle ,\ell\in\ta\right\} \subset\mathbb{R}^{\mathcal{S}}
\]
is an algebra of functions on streams. Equipping $\mathcal{S}$ with
bounded variation topology it follows that any continuous function
of the stream $f\in C\left(\mathcal{S},\mathbb{R}\right)$ can be
approximated by a $\ell\in\ta$ 
\[
f\left(\cdot\right)\approx\left\langle \ell,\Phi\left(\cdot\right)\right\rangle ,
\]
uniformly over compact sets in $\mathcal{S}$. In the terminology
of machine learning \cite{micchelli2006universal}, the features $\left(\Phi_{w}\left(\sigma\right)\right)_{w}$
are ``universal'' and this allows to use linear learning algorithms;
we provide the full details in the appendix \ref{sec:top stream}
and \ref{sec:Hopf}.

\begin{rem}
Algebraically it makes more sense to introduce $\tb$ as the dual
of $\ta$, see Appendix \ref{sec:Hopf}. However, in a learning context
we regard $\Phi\left(\sigma\right)\in\tb$ as features and learn about
$\sigma$ by linear functionals $\ell\in\ta$.
\end{rem}

\begin{rem}
The coefficients $\left(\Phi_{w}\left(\sigma\right)\right)_{w}$ carry
redundant information: a simple calculation shows that $\left(\Phi_{a}\left(\sigma\right),\Phi_{b}\left(\sigma\right),\Phi_{ab}\left(\sigma\right)-\Phi_{ba}\left(\sigma\right)\right)_{a,b\in\letters}$
already completely determines $\left(\Phi_{w}\left(\sigma\right)\right)_{w\in\alph,\left|w\right|\leq2}$
. The reason is that $\Phi\left(\sigma\right)$ lives in a nonlinear
subset of $\tb$; in the case $p\left(\lambda,a\right)=\exp\left(\lambda a\right)$
this is the free Lie group generated by $\left|\letters\right|$ variables
and the previous observation is simply that we can work in the Lie
algebra instead of the Lie group. However, a classic result of Bourbaki
is that the dimension of the Lie algebra is $O\left(\frac{\left|\letters\right|^{M}}{M}\right)$
as $M\rightarrow\infty$, i.e.~this does not kill the exponential
growth. Nevertheless, our main sketch can be immediately applied to
the Lie algebra which leads to a reduction of computational cost by
a constant for the price of more complicated algebraic objects.
\end{rem}

\subsection{Streams, paths and scaling limits}

Given a stream $\sigma\in\mathcal{S}$ we can identify it as a path
in the free vector space spanned by the letters $\letters$. That
is, identify the set of letters $\letters$ as \foreignlanguage{british}{ONB}
basis for $\mathbb{R}^{\left|\letters\right|}$ and an event $\sigma_{i}=\left(\lambda_{i},a_{i}\right)$
as the vector $\lambda_{i}a_{i}\in\mathbb{R}^{\left|\letters\right|}$
. Consequently $\sigma=\left(\lambda_{i},a_{i}\right)_{i=1}^{L}$
can be seen as the continuous path $\gamma$ in the free vector space
spanned by $\letters$ given by the Donsker embedding
\[
t\mapsto\gamma\left(t\right):=\sigma{}_{\left\lceil Lt\right\rceil }\left(Lt-\left\lfloor Lt\right\rfloor \right)+\sum_{i=1}^{\left\lfloor Lt\right\rfloor }\sigma_{i},
\]
That is we inject $\mathcal{S}\hookrightarrow C\left(\left[0,1\right],\mathbb{R}^{\left|\letters\right|}\right)$
by linear interpolation, see for example Figure \eqref{fig:stream}.
This gives a topology on streams $\mathcal{S}$ that captures the
intuitive notion of streams as being similar if they have similar
events and length. Further, it allows to study what happens if the
number of events goes to infinity, $L\rightarrow\infty$, and it gives
another interpretation to our features $\Phi\left(\sigma\right)$,
namely 
\begin{equation}
\Phi_{a_{1}\cdots a_{M}}\left(\sigma\right)=\int_{0\leq t_{1}\leq\cdots\leq t_{M}\leq1}d\gamma^{a_{1}}\left(t_{1}\right)\cdots d\gamma^{a_{M}}\left(t_{M}\right).\label{eq:sigs}
\end{equation}
We give full details in the Appendix \ref{sec:Hopf}. From this point
of view, our sketch chooses a random but very structured linear map
$H$ from $\mathbb{R}^{\left|\letters\right|}$ to $\mathbb{R}^{d}$
for $d\ll\left|\letters\right|$ to turn $\gamma\in C\left(\left[0,1\right],\mathbb{R}^{\left|\letters\right|}\right)$
into $\hat{\gamma}=\left(H\gamma\left(t\right)\right)\in C\left(\left[0,1\right],\mathbb{R}^{d}\right)$
that gives good estimates for the large coordinates of (\ref{eq:sigs}).

\begin{figure}[h!] 
\begin{tikzpicture} 
\draw[thick, ->] (0,0) -- (3.8,0) node[anchor=north] {$a$};  
\draw[thick, ->] (0,0) -- (0,3.8) node[anchor=east] {$b$};  

\draw[dotted] (0,2) -- (3.8,2); 
\draw[dotted] (0,3) -- (3.8,3); 
\draw[dotted] (0,1) -- (3.8,1); 
\draw[dotted] (1,0) -- (1,3.8); 
\draw[dotted] (2,0) -- (2,3.8); 
\draw[dotted] (3,0) -- (3,3.8);
\draw[->] (0,0) -- (1,0);
\draw[->] (1,0) -- (1,1.5);
\draw[->] (1,1.5) -- (1,2.5);
\draw[->] (1,2.5) -- (3,2.5);
\end{tikzpicture}
\caption{\label{fig:stream}The stream $\sigma=\left(\left(1,a\right),\left(1.5,b\right),\left(1,b\right),\left(2,a\right)\right)$ as a path in the free vector spaced spanned by the letters $a,b$.}
\end{figure}

\section{Sketching the order of events\label{sec:Signature-sketches}}

The algebraic construction of the feature map 
\[
\Phi:\mathcal{S}\rightarrow\mathcal{F},\quad\sigma=\left(\sigma_{i}\right)_{i=1}^{L}\mapsto\Phi\left(\sigma\right)=\prod_{i=1}^{L}p\left(\sigma_{i}\right)
\]
based on a event map $p:\mathcal{E}\rightarrow\tb$ gives theoretical
guarantees that address Points (\ref{enu:linearizes}),(\ref{enu:features algebra}),(\ref{enu:(robust-under-noise})
mentioned in the introduction, see Appendix \ref{sec:top stream}
and \ref{sec:Hopf} for details. In this section, we present our main
result that addresses the computational aspect, Point (\ref{enu:computed}):
although $\tb$ is infinite dimensional, it is graded by word length.
Analogous to the classic case of polynomials in commuting variables
as features, a sensible approach is to truncate at a given degree
$M$, i.e.~to consider 
\[
\Phi_{M}\left(\sigma\right):=\sum_{\left|w\right|\leq M}\Phi_{w}\left(\sigma\right)w.
\]
(in supervised learning tasks, $M$ is a parameter that must be chosen
as to minimize the generalization error, e.g.~in a supervised setting
by cross-validation). Unfortunately, the combinatorial explosion of
coordinates is 
\[
\left|\left\{ \Phi_{w}\left(\sigma\right):\left|w\right|\leq M\right\} \right|=O\left(\left|\letters\right|^{M}\right)
\]
and for many streaming applications $\left|\letters\right|$ is so
large that already $M=1$ is infeasible. Our main theorem is
\begin{thm}
\label{thm:main}Let $\sigma\in\mathcal{S}$. For every $\epsilon,\delta>0$
there exists a $\tb$-valued random variable $\hat{\Phi}\left(\sigma\right)$
such that 
\[
\mathbb{P}\left(\frac{\left|\hat{\Phi}_{w}\left(\sigma\right)-\Phi_{w}\left(\sigma\right)\right|}{\left\Vert \Phi\left(\sigma\right)\right\Vert _{1,\left(M\right)}}>\epsilon\right)<\delta,\quad\forall w\in\mbox{\ensuremath{\alph}}\text{ with }\left|w\right|\leq M
\]
and we can compute the set of coordinates 
\[
\left\{ \hat{\Phi}_{w}\left(\sigma\right):\left|w\right|\leq M\right\} 
\]
using $O\left(\epsilon^{-M}\log\frac{1}{\delta}\right)$ memory units\footnote{A memory units stores a positive real number. For implementations
the usual considerations and additional cost to deal with rounding
errors, floating numbers, etc.~apply.} ,$\left\lceil -\log\delta\right\rceil \log\left|\letters\right|$
random bits and a single pass over $\sigma$ using Algorithm \ref{alg:Signature-sketch}.
\end{thm}
We can express the error estimate also in terms of the length of
$\sigma$, $\left\Vert \sigma\right\Vert =\sum_{i=1}^{L}\left|\lambda_{i}\right|$.
\begin{cor}
Let $\sigma$ and $\hat{\Phi}\left(\sigma\right)$ be as above. Then
\[
\mathbb{P}\left(\left|\hat{\Phi}_{w}\left(\sigma\right)-\Phi_{w}\left(\sigma\right)\right|>\epsilon\frac{\left\Vert \sigma\right\Vert ^{M}}{M!}\right)<\delta
\]
for every word $w$ of length $\left|w\right|=M$.
\end{cor}
Note
\begin{itemize}
\item the appearance of the $1$-norm $\left\Vert \Phi\left(\sigma\right)\right\Vert _{1,\left(M\right)}$.
Above is a very good estimate for ``heavy hitter patterns'', i.e.~for
power law type distributions in the coordinates $\left(\Phi_{w}\left(\sigma\right)\right)_{w\in\alph}$,
\item the alphabet size $\left|\letters\right|$ appears only in a logarithm
in the computational complexity,
\item that applied with $M=1$, above calculates $\left\{ \hat{\Phi}_{a}\left(\sigma\right):a\in\letters\right\} $
and Algorithm \ref{alg:Signature-sketch} is then simply the count-min
frequency sketch of Cormode--Muthukrishnan \cite{cormode2005improved}.\end{itemize}
\begin{rem}
\label{rem: tursntile}Above can be modified for cash-register streams
with simple modifications: the first option is to replace the coordinate-wise
minimum by a median; the second option is to use a count-sketch \cite{charikar2002finding,charikar2004finding}
giving an additional factor of $\epsilon^{-1}$ in space complexity.
Similarly, we can replace the $1$-norm by other norms in analogy
with standard sketches, $M=1$. This leads to higher computational
complexity, analogous to the letter frequency case, see \cite{muthukrishnan2005data}.
\end{rem}
The rest of Section \ref{sec:Signature-sketches} is devoted to the
proof of Theorem \ref{thm:main}. It is a simple but instructive exercise
to run through the remainder of this section for the special case
$M=1$ (not patterns, just frequencies) and see how this recovers
the standard proof the count-min sketch \cite{cormode2005improved}.

\subsection{Hashed streams\label{sec:hashed-streams}}

Fix a function $h:\letters\rightarrow\alphsmall$, where $\alphsmall$
is finite set with $\left|\alphsmall\right|<\left|\letters\right|$.
We study how much information we can recover about $\Phi\left(\sigma\right)$
if we only observe the hashed stream $\sigma^{h}:=\left(\lambda_{i},h\left(a_{i}\right)\right)_{i}$
of the stream $\sigma=\left(\lambda_{i},a_{i}\right)_{i}$. To do
so, we work with three objects: 
\[
\Phi\left(\sigma\right)\in\tb,\quad\Phi\left(\sigma^{h}\right)\in\tbsmall,\quad\Phi_{h}\left(\sigma\right)\in\tb,
\]
the\emph{ original features }$\Phi\left(\sigma\right)$, the \emph{features
}$\Phi\left(\sigma^{h}\right)$\emph{ of the hashed stream $\sigma^{h}$}
and the third object, denoted $\Phi_{h}\left(\sigma\right)$, is the
\emph{``pull-back''} of $\Phi\left(\sigma^{h}\right)$ to $\tb=\mathbb{R}\left\langle \left\langle \letters\right\rangle \right\rangle $. 
\begin{defn}
Let $\alphsmall$ be a finite set and $h:\letters\rightarrow\alphsmall$.
For $\sigma=\left(\sigma_{i}\right)$ with $\sigma_{i}=\left(\lambda_{i},a_{i}\right)\in\mathbb{R}_{>0}\times\letters$,
define 
\[
\sigma^{h}:=\left(\sigma_{i}^{h}\right)\text{ with }\sigma_{i}^{h}:=\left(\lambda_{i},h\left(a_{i}\right)\right)\in\mathbb{R}_{>0}\times\alphsmall.
\]
We call $\sigma^{h}$ the $h$-hash of $\sigma$. 
\end{defn}
We are interested in the situation where $\alphsmall$ is much smaller
than the original alphabet $\letters$ such that we can afford the
computational cost to calculate $\Phi\left(\sigma^{h}\right)\in\tbsmall$.
Given $\Phi\left(\sigma^{h}\right)$ we ``pull back'' these features
to $\tb$ to get an approximation of $\Phi\left(\sigma\right)$.
\begin{defn}
Define $\Phi_{h}\left(\sigma\right)\in\tb$ as 
\[
\left\langle \Phi_{h}\left(\sigma\right),w\right\rangle :=\left\langle \Phi\left(\sigma^{h}\right),h\left(w\right)\right\rangle \text{ for }w\in\alph
\]
where $h\left(w\right):=h\left(a_{1}\right)\cdots h\left(a_{M}\right)\in\alphsmall^{\star}$
for $w=a_{1}\cdots a_{M}\in\alph$.
\end{defn}
If $\left|\alphsmall\right|<\left|\letters\right|$ then $h$ is not
injective, $h\left(a\right)=h\left(b\right)$ for $a\neq b$. These
collisions are the reason $\Phi_{h}\left(\sigma\right)$ overestimates
$\Phi\left(\sigma\right)$. 
\begin{lem}
\label{lem:overestimate}For any $h\in\alphsmall^{\letters}$ we have
\[
\left\langle \Phi\left(\sigma\right),w\right\rangle \leq\left\langle \Phi_{h}\left(\sigma\right),w\right\rangle \text{\text{ for all }}w\in\alph.
\]

\end{lem}
Lemma follows directly from the formulas for $\Phi\left(\sigma\right)$
given in Theorem \ref{thm:sig hopf algebra}. In fact, we get the
following explicit expression for this bias.
\begin{prop}
Let $h\in\alphsmall^{\letters}$ and $\sigma=\left(\lambda_{i},a_{i}\right)_{i=1}^{L}\in\mathcal{S}$.
Then
\[
\Phi_{h}\left(\sigma\right)=\Phi\left(\sigma\right)+b\text{ and }\left\langle b,w\right\rangle =\begin{cases}
\sum_{\left(1\right)}\frac{\lambda\left(\boldsymbol{i}\right)}{\boldsymbol{i}!}1_{h\left(a\left(\boldsymbol{i}\right)\right)=h\left(w\right)} & \text{if }p\left(\lambda,a\right)=\exp\left(\lambda a\right),\\
\sum_{\left(2\right)}\lambda\left(\boldsymbol{i}\right)1_{h\left(a\left(\boldsymbol{i}\right)\right)=h\left(w\right)} & \text{ if }p\left(\lambda,a\right)=1+\lambda a.
\end{cases}
\]
where the sum $\sum_{\left(1\right)}$ is taken over all $\boldsymbol{i}=\left(i_{1},\ldots,i_{M}\right)$
such that $i_{1}\leq\cdots\leq i_{M}$ and $a\left(\boldsymbol{i}\right)\neq w$
where $a\left(\boldsymbol{i}\right):=a_{i_{1}}\cdots a_{i_{M}}$;
for the sum $\sum_{\left(2\right)}$ we additionally assume $i_{1}<\cdots<i_{M}$.\end{prop}
\begin{proof}
If $p\left(\lambda,a\right)=\exp\left(\lambda,a\right)$ (the other
case follows by similar arguments) then we know that $\left\langle \Phi\left(\sigma\right),w\right\rangle =\sum\frac{\lambda\left(\boldsymbol{i}\right)}{\boldsymbol{i}!}$
where the sum is over $\boldsymbol{i}$, $i_{1}\leq\cdots\leq i_{M}$
such that $a\left(\boldsymbol{i}\right)=w$ . Now from Theorem \ref{thm:sig hopf algebra}
we know that 
\[
\left\langle \Phi_{h}\left(\sigma\right),w\right\rangle \equiv\left\langle \Phi\left(\sigma^{h}\right),h\left(w\right)\right\rangle =\sum\frac{\lambda\left(\boldsymbol{i}\right)}{\boldsymbol{i}!}1_{h\left(a\left(\boldsymbol{i}\right)\right)=h\left(w\right)}
\]
with the sum over all $\boldsymbol{i}$, $i_{1}\leq\cdots\leq i_{M}$.
The statement follows by splitting $1_{h\left(a\left(\boldsymbol{i}\right)\right)=h\left(w\right)}=1_{a\left(\boldsymbol{i}\right)=w}+1_{a\left(\boldsymbol{i}\right)\neq w}1_{h\left(a\left(\boldsymbol{i}\right)\right)=h\left(w\right)}$.\end{proof}
\begin{cor}
\label{cor:hashed sig equals sig plus bias}Let $H$ be a $\alphsmall^{\letters}$-valued
random variable. For $\sigma=\left(\lambda_{i},a_{i}\right)_{i=1}^{L}\in\mathcal{S}$
we have
\begin{eqnarray*}
\mathbb{E}\left[\Phi_{H}\left(\sigma\right)\right] & = & \Phi\left(\sigma\right)+\bias\left(\sigma\right),\\
\text{ where }\left\langle \bias\left(\sigma\right),w\right\rangle  & = & \begin{cases}
\sum_{\left(1\right)}\frac{\lambda\left(\boldsymbol{i}\right)}{\boldsymbol{i}!}\mathbb{P}\left(H\left(a\left(\boldsymbol{i}\right)\right)=H\left(w\right)\right) & \text{if }p\left(\lambda,a\right)=\exp\left(\lambda a\right),\\
\sum_{\left(2\right)}\lambda\left(\boldsymbol{i}\right)\mathbb{P}\left(H\left(a\left(\boldsymbol{i}\right)\right)=H\left(w\right)\right) & \text{ if }p\left(\lambda,a\right)=1+\lambda a.
\end{cases}
\end{eqnarray*}
As a consequence, if $\mathbb{P}\left(H\left(w\right)=H\left(w^{\prime}\right)\right)\leq q$
for words $w\neq w^{\prime}$, $\left|w\right|=\left|w^{\prime}\right|=M$
then 
\[
0\leq\left\langle \bias\left(\sigma\right),w\right\rangle \leq q\left\Vert \Phi\left(\sigma\right)\right\Vert _{1,\left(M\right)}\text{ for }w\in\alph\text{ with }\left|w\right|=M.
\]
\end{cor}
\begin{proof}
Note that $\sum_{\boldsymbol{i}:i_{1}\leq\cdots\leq i_{M}}\frac{\lambda\left(\boldsymbol{i}\right)}{\boldsymbol{i}!}=\sum_{\boldsymbol{i}:i_{1}\leq\cdots\leq i_{M}}\frac{\lambda\left(\boldsymbol{i}\right)}{\boldsymbol{i}!}\sum_{w}1_{a\left(\boldsymbol{i}\right)=w}=\sum_{w:\left|w\right|=M}\Phi_{w}\left(\sigma\right)\equiv\left\Vert \Phi\left(\sigma\right)\right\Vert _{1,\left(M\right)}$,
hence the statement follows since $\mathbb{E}\left[\sum_{\left(1\right)}\frac{\lambda\left(\boldsymbol{i}\right)}{\boldsymbol{i}!}1_{H\left(a\left(\boldsymbol{i}\right)\right)=H\left(w\right)}\right]\leq q\sum_{\boldsymbol{i}:i_{1}\leq\cdots\leq i_{M}}\frac{\lambda\left(\boldsymbol{i}\right)}{\boldsymbol{i}!}$.
\end{proof}

\subsection{Combining independent hashes}

We get tail estimates for $\Phi_{H}\left(\sigma\right)$ for a randomly
chosen $H$.
\begin{prop}
\label{prop: markov}Let $M\in\mathbb{N}$ and $H$ be a $\alphsmall^{\letters}$-valued
random variable. Then
\begin{enumerate}
\item $\mathbb{P}\left(\left\Vert \Phi_{H}\left(\sigma\right)-\Phi\left(\sigma\right)\right\Vert _{1,M}>x\right)\leq\frac{\left\Vert \bias\left(\sigma\right)\right\Vert _{1,M}}{x}$
for $x>\left\Vert \bias\left(\sigma\right)\right\Vert _{1,M}$,
\item $\mathbb{P}\left(\left\langle \Phi\left(\sigma\right),w\right\rangle \in\left[\left\langle \Phi_{H}\left(\sigma\right),w\right\rangle -x,\left\langle \Phi_{H}\left(\sigma\right),w\right\rangle \right]\right)\geq1-\frac{\left\langle \bias\left(\sigma\right),w\right\rangle }{x}$
\textup{for $w\in\alph$, $x>\left\langle \bias\left(\sigma\right),w\right\rangle $}
\end{enumerate}
As a consequence, if $\mathbb{P}\left(H\left(w\right)=H\left(v\right)\right)\leq q$
for words $w\neq v$, $\left|w\right|=\left|v\right|=M$ then
\[
\mathbb{P}\left(\left\langle \Phi\left(\sigma\right),w\right\rangle \in\left[\left\langle \Phi_{H}\left(\sigma\right),w\right\rangle -2q\left\Vert \Phi\left(\sigma\right)\right\Vert _{1,\left(M\right)},\left\langle \Phi_{H}\left(\sigma\right),w\right\rangle \right]\right)\geq\frac{1}{2}.
\]
\end{prop}
\begin{proof}
We apply the Markov inequality and Corollary \ref{cor:hashed sig equals sig plus bias}
to 
\[
\left\Vert \Phi_{H}\left(\sigma\right)-\Phi\left(\sigma\right)\right\Vert _{1,M}=\sum_{w:\left|w\right|\leq M}\left\langle \Phi_{H}\left(\sigma\right)-\Phi\left(\sigma\right),w\right\rangle 
\]
to get 
\[
\mathbb{P}\left(\left\Vert \Phi_{H}\left(\sigma\right)-\Phi\left(\sigma\right)\right\Vert _{1,M}>x\right)\leq\frac{\left\Vert \bias\left(\sigma\right)\right\Vert _{1,M}}{x}.
\]
Similarly, applying Markov's inequality coordinate-wise shows 
\[
\mathbb{P}\left(\left|\left\langle \Phi_{H}\left(\sigma\right),w\right\rangle -\left\langle \Phi\left(\sigma\right),w\right\rangle \right|>x\right)\leq\frac{\left\langle \bias\left(\sigma\right),w\right\rangle }{x}.
\]
Hence, 
\[
\mathbb{P}\left(\left\langle \Phi\left(\sigma\right),w\right\rangle \in\left[\left\langle \Phi_{H}\left(\sigma\right),w\right\rangle -x,\left\langle \Phi_{H}\left(\sigma\right),w\right\rangle \right]\right)\geq1-\frac{\left\langle \bias\left(\sigma\right),w\right\rangle }{x}.
\]
By Corollary \ref{cor:hashed sig equals sig plus bias}, 
\[
\left\langle \bias\left(\sigma\right),w\right\rangle =\begin{cases}
\sum_{\substack{\left(1\right)}
}\frac{\lambda\left(\boldsymbol{i}\right)}{\boldsymbol{i}!}\mathbb{P}\left(H\left(a\left(\boldsymbol{i}\right)\right)=H\left(w\right)\right) & \text{if }p\left(\lambda,a\right)=\exp\left(\lambda a\right),\\
\sum_{\left(2\right)}\lambda\left(\boldsymbol{i}\right)\mathbb{P}\left(H\left(a\left(\boldsymbol{i}\right)\right)=H\left(w\right)\right) & \text{ if }p\left(\lambda,a\right)=1+\lambda a.
\end{cases}
\]
and $\left\langle \bias\left(\sigma\right),w\right\rangle \leq q\left\Vert \Phi\left(\sigma\right)\right\Vert _{1,\left(M\right)}$.
This shows 
\[
\mathbb{P}\left(\left\langle \Phi\left(\sigma\right),w\right\rangle \in\left[\left\langle \Phi_{H}\left(\sigma\right),w\right\rangle -x,\left\langle \Phi_{H}\left(\sigma\right),w\right\rangle \right]\right)\geq1-q\frac{\left\Vert \Phi\left(\sigma\right)\right\Vert _{1,\left(M\right)}}{x}\text{ for }x>q\left\Vert \Phi\left(\sigma\right)\right\Vert _{1,\left(M\right)}
\]
The statement follows by choosing $x=2q\left\Vert \Phi\left(\sigma\right)\right\Vert _{1,\left(M\right)}$.
\end{proof}
Since $\Phi_{H}$ overestimates $\Phi$, Lemma \ref{lem:overestimate},
we can take independent copies of $H$ and take the coordinate-wise
minimum to combine these estimators.
\begin{prop}
\label{lem:ind hashes}Let $H_{1},\ldots H_{r}$ be independently
and identically distributed $\alphsmall^{\letters}$-valued random
variables. Assume there exists a $q>0$ such that for 
\[
\mathbb{P}\left(H_{j}\left(w\right)=H_{j}\left(w^{\prime}\right)\right)\leq q\text{ for }j=1,\ldots,r\text{ and }w,w^{\prime}\in\alph,w\neq w.
\]
Then for any $w\in\alph$
\begin{enumerate}
\item \textup{$\mathbb{P}\left(\min_{j=1,\ldots,r}\left\langle \Phi_{H_{j}}\left(\sigma\right),w\right\rangle -\left\langle \Phi\left(\sigma\right),w\right\rangle >2q\left\Vert \Phi\left(\sigma\right)\right\Vert _{1,\left(M\right)}\right)\leq2^{-r}$,}
\item \textup{$\mathbb{P}\left(\left\langle \Phi\left(\sigma\right),w\right\rangle \in\left[\min_{j}\left\langle \Phi_{H_{j}}\left(\sigma\right),w\right\rangle -2q\left\Vert \Phi\left(\sigma\right)\right\Vert _{1,\left(M\right)},\min_{j}\left\langle \Phi_{H_{j}}\left(\sigma\right),w\right\rangle \right]\right)\geq1-2^{-r}$}
\end{enumerate}
where $M=\left|w\right|$ denotes the length of the word $w\in\alph$.\end{prop}
\begin{proof}
Using Proposition \ref{prop: markov} with $x=2q\left\Vert \Phi\left(\sigma\right)\right\Vert _{1,\left(M\right)}$,
independence and $\left\langle \Phi_{H_{j}}\left(\sigma\right),w\right\rangle \geq\left\langle \Phi\left(\sigma\right),w\right\rangle $
yields 
\begin{eqnarray*}
\mathbb{P}\left(\min_{i=1,\ldots,r}\left\langle \Phi_{H_{i}}\left(\sigma\right),w\right\rangle -\left\langle \Phi\left(\sigma\right),w\right\rangle >x\right) & = & \prod_{i=1}^{r}\mathbb{P}\left(\left\langle \Phi_{H_{i}}\left(\sigma\right),w\right\rangle -\left\langle \Phi\left(\sigma\right),w\right\rangle >x\right)\\
 & = & \prod_{i=1}^{r}\mathbb{P}\left(\left\langle \Phi\left(\sigma\right),w\right\rangle \notin\left[\left\langle \Phi_{H_{i}}\left(\sigma\right),w\right\rangle -x,\left\langle \Phi_{H_{i}}\left(\sigma\right),w\right\rangle \right]\right)\\
 & \leq & 2^{-r}.
\end{eqnarray*}

\end{proof}
Unfortunately, sampling uniformly from $\alphsmall^{\letters}$ is
too expensive: there are $\left|\alphsmall\right|^{\left|\letters\right|}$
such functions to choose from, so specifying an element of $\alphsmall^{\letters}$
requires $\left|\letters\right|\log\left|\alphsmall\right|$ bits
of memory. This is prohibitively expensive for our applications where
$\letters$ is a very large set. Fortunately, this a classic topic
in computer science solved by universal hashes.

\subsection{Universal hashes}

We sample uniformly from a subset\emph{ $\mathcal{H}\subset\alphsmall^{\letters}$}.
Our goal is to construct a set $\mathcal{H}$ such that the assumptions
of Proposition \ref{lem:ind hashes} are met, that is we need to bound
$\mathbb{P}\left(H\left(w\right)\neq H\left(w\right)\right)$ for
$w\neq w^{\prime}$. Constructing such families of functions is not
trivial but a classic topic in hashing.
\begin{defn}
Fix $\mathcal{H}\subset\alphsmall^{\letters}$. Let $H$ be a chosen
uniformly at random from $\mathcal{H}$. We call $\mathcal{H}$ a
2-universal family of hash functions if
\[
\mathbb{P}\left(H\left(a\right)=H\left(b\right)\right)\leq\left|\alphsmall\right|^{-1}
\]
for $a\neq b$. 
\end{defn}
Since $\mathbb{P}\left(H\left(a_{1}\cdots a_{M}\right)=H\left(b_{1}\cdots b_{M}\right)\right)\leq\sup_{i}\mathbb{P}\left(H\left(a_{i}\right)=H\left(b_{i}\right)\right)$
this is enough to apply Proposition \ref{lem:ind hashes}.
\begin{example}
Let $\letters=\left\{ 1,\ldots,m\right\} $, $\alphsmall=\left\{ 1,\ldots,n\right\} $.
For any prime $p\geq m$, the set 
\[
\mathcal{H}=\left\{ h_{a,b}\lvert h_{a,b}\left(x\right)=\left(\left(\left(ax+b\right)\mod p\right)\mod n\right),1\leq a\leq p-1,0\leq b\leq p-1\right\} \subset\alphsmall^{\letters}
\]
is $2$-universal. Hence choosing a random element of $\mathcal{H}$
requires $2\log p$ random bits.
\end{example}

\subsection{Sketching features}

Since for every random hash, the estimator $\Phi_{H}\left(\sigma\right)$
overestimates $\Phi\left(\sigma\right)$ we simply take the minimum
to get a better estimator $\hat{\Phi}\left(\sigma\right)$ for $\Phi\left(\sigma\right)$.
The specific choice of alphabet size and number of hashes will become
clear from the probabilistic bounds given in Proposition \ref{thm:complexity}.
\begin{defn}
For $\epsilon,\delta>0$ let $\mathcal{H}\subset\alphsmall^{\letters}$
be a $2$-universal family into an alphabet $\alphsmall$ of cardinality
$\left|\alphsmall\right|=\left\lceil 2\epsilon^{-1}\right\rceil $.
Draw $r=\left\lceil \log\delta\right\rceil $ elements $H_{1},\ldots,H_{r}$
independently and uniformly from $\mathcal{H}$ and define the random
map 
\[
\sigma\mapsto\hat{\Phi}\left(\sigma\right)\in\tb\text{ by }\hat{\Phi}_{w}\left(\sigma\right):=\min_{i=1,\ldots,M}\left\langle \Phi_{H_{i}}\left(\sigma\right),w\right\rangle \text{ for }w\in\alph.
\]
We call $\hat{\Phi}\left(\sigma\right)$ the \emph{$\left(\epsilon,\delta\right)$-count-min-sketch}
of $\Phi\left(\sigma\right)$.
\end{defn}
Applying the previous estimates yields
\begin{prop}
\label{thm:complexity}Let $\hat{\Phi}\left(\sigma\right)$ be the
$\left(\epsilon,\delta\right)$-count-min-sketch of $\Phi\left(\sigma\right)$.
Then 

\[
\mathbb{P}\left(\Phi_{w}\left(\sigma\right)\in\left[\hat{\Phi}_{w}\left(\sigma\right)-\epsilon\left\Vert \Phi\left(\sigma\right)\right\Vert _{1,\left(M\right)},\hat{\Phi}_{w}\left(\sigma\right)\right]\right)\geq1-\delta\text{ for all }w\in\alph\text{ with }\left|w\right|=M.
\]
Algorithm \ref{alg:Signature-sketch} computes \textup{for given $\left(\epsilon,\delta,M\right)$
the set of coordinates 
\[
\left\{ \hat{\Phi}_{w}:\left|w\right|\leq M\right\} 
\]
}with a single pass over $\sigma$ using $O\left(\frac{1}{\varepsilon^{M}}\log\frac{1}{\delta}\right)$
memory units and $\left\lceil -\log\delta\right\rceil \log\left|\letters\right|$
random bits. \end{prop}
\begin{proof}
By universality of $\mathcal{H}$ we have $\mathbb{P}\left(H\left(w\right)=H\left(w^{\prime}\right)\right)\leq\sup_{a,b\in\letters,a\neq b}\mathbb{P}\left(H\left(a\right)=H\left(b\right)\right)\leq\left|\alphsmall\right|^{-1}$.
Thus the assumptions of Proposition \ref{lem:ind hashes} are met
with $q=2^{-1}\epsilon$. Using the $2$-universal hash family, we
can store $r$ hash functions in $2r\log\left|\letters\right|$ random
bits. Each $\left\langle \Phi_{H}\left(\sigma\right),w\right\rangle $
for $\left|w\right|\leq M$ requires $O\left(\left|\alphsmall\right|^{M}\right)=O\left(\epsilon^{-M}\right)$
of storage, thus for $r\equiv\left\lceil -\log\delta\right\rceil $
hashes we need $O\left(\frac{1}{\varepsilon^{M}}\log\frac{1}{\delta}\right)$
memory units. 
\end{proof}

\section{\label{sec:Algorithms,-Applications-and}Applications: pattern queries
and heavy hitter patterns}

\subsection{Pattern queries}

Theorem \ref{thm:main} allows to estimate patterns in the stream
$\sigma$. For the event map $p\left(\lambda,a\right)=1+\lambda a$,
recall that the resulting features are 
\[
\Phi_{w}\left(\sigma\right)=\sum_{\left(1\right)}\lambda_{i_{1}}\cdots\lambda_{i_{M}}
\]
with the sum taken over all $i_{1}<\cdots<i_{M}$ such that $a_{i_{1}}\cdots a_{i_{M}}=w$.
Thus they count how often a word $w$ appears within the stream $\sigma$,
weighted by the counter increases. Algorithm \ref{alg:Signature-sketch}
calculates the estimator $\hat{\Phi}^{\left(M\right)}\left(\sigma\right)$
that is given in Theorem \ref{thm:main}. The same applies with $p\left(\lambda,a\right)=\exp\left(\lambda,a\right)$
though in this case we allow for $i_{1}\leq\cdots\leq i_{M}$ and
account for this with a factorial $\frac{1}{\boldsymbol{i}!}$. .

\subsection{Patters of heavy hitters}

Fix a threshold $\rho>0$ and $M\in\mathbb{N}$. Denote 
\[
\mathcal{H}\left(\sigma\right)=\left\{ a\in\letters:\Phi_{a}\left(\sigma\right)\geq\rho\right\} 
\]
where we recall the definition $\left\Vert \Phi\left(\sigma\right)\right\Vert _{1,M}=\sum_{w:\left|w\right|=M}\Phi_{w}\left(\sigma\right)$.
The standard (count-min) sketch produces a random set $\hat{\mathcal{H}}$
such that $\mathcal{H}\subset\mathcal{\hat{H}}$ and $\mathcal{\hat{H}}\backslash\mathcal{H}$
is small. Recall one of our motivations in the introduction was to
learn effects/functions of streams 
\[
f:\mathcal{S}\rightarrow\mathbb{R}
\]
from the sketch $\hat{\Phi}\left(\sigma\right)$. A common scenario
is that $f\left(\sigma\right)$ is well approximated by the stream
of heavy hitters, i.e. $f\left(\sigma\right)\approx g\left(\sigma^{\mathcal{H}}\right)$
where $g$ denotes simply the restriction of $f$ to streams in letters
in $\mathcal{H}\equiv\mathcal{H}\left(\sigma\right)$, and $\sigma^{\mathcal{H}}=\left(a_{i},\lambda_{i}\right)_{a_{i}\in\mathcal{H}}$.
Combining the universality $\Phi$, Proposition \ref{prop:universal},
with the sketch of Theorem \ref{thm:main} yields 

\begin{eqnarray*}
f\left(\sigma\right) & \approx & g\left(\sigma^{\mathcal{H}}\right)\\
 & \approx & \left\langle \ell,\Phi\left(\sigma^{\mathcal{H}}\right)\right\rangle \\
 & = & \sum_{w\in\left(\mathcal{H}\left(\sigma\right)\right)^{\star}}\ell_{w}\Phi_{w}\left(\sigma^{\mathcal{H}}\right).
\end{eqnarray*}
Thus we can train a linear learning algorithm with $\left(\hat{\Phi}_{w}\left(\sigma\right)\right)_{w\in\mathcal{H}\left(\sigma\right)^{\star}}$
as features to find $\ell\in\ta$. 
\begin{defn}
Let $\rho>0$ and $M\geq1$. We call 
\[
\mathcal{H}\left(\sigma\right)=\left\{ a\in\letters:\Phi_{a}\left(\sigma\right)>\rho\right\} 
\]
the set of heavy hitter letters of threshold $\rho$ for the stream
$\sigma$. We call 
\[
\mathcal{H}_{M}\left(\sigma\right)=\left\{ w\in\mathcal{H}\left(\sigma\right)^{\star}:\Phi_{w}\left(\sigma\right)>\rho^{\left|w\right|}\right\} 
\]
the set of heavy patterns of length $M$ of threshold $\rho$ of the
stream $\sigma$.\end{defn}
\begin{rem}
A subtlety is that coordinates of words of different lengths scale
differently: if we replace $\sigma=\left(\lambda_{i},a_{i}\right)$
with $\sigma^{c}=\left(c\lambda_{i},a_{i}\right)$ for some $c>0$
then $\Phi_{w}\left(\sigma^{c}\right)=c^{\left|w\right|}\Phi_{w}\left(\sigma\right)$).
Thus it is natural to scale the threshold accordingly as done in above
definition.
\end{rem}
To carry out the above learning method (find the heavy hitter patterns
$\mathcal{H}_{M}$, train a learning algorithm on $\left(\hat{\Phi}_{w}\left(\sigma\right)\right)_{w\in\mathcal{H}_{M}}$)
we need to compute the set of heavy hitter patterns $\mathcal{H}_{M}$.
We could use Theorem \ref{thm:main} and query the sketch for all
$O\left(\left|\letters\right|^{M}\right)$ coordinates but this is
too expensive. Instead, as in the $M=1$ case, one can approximate
the list $\mathcal{H}_{M}$ while parsing the stream.
\begin{thm}
\label{thm:HH}For given threshold $\rho>0$, Algorithm \ref{alg:Signature-sketch-1-2}
computes for a given stream $\sigma$ a random set of words $\hat{\mathcal{H}}$
such that 
\begin{itemize}
\item $\mathcal{H}_{M}\left(\sigma\right)\subset\hat{\mathcal{H}}$,
\item If $w\notin\mathcal{H}_{M}\left(\sigma\right)$ with $\Phi_{w}\left(\sigma\right)<\rho-\epsilon\left\Vert \Phi\left(\sigma\right)\right\Vert _{1,\left(M\right)}$
then $\mathbb{P}\left(w\in\hat{\mathcal{H}}\right)<\delta$.
\end{itemize}
Algorithm \ref{alg:Signature-sketch-1-2} uses one pass over $\sigma$,
$O\left(\epsilon^{-M}\log\frac{1}{\delta}\right)$ memory units and
$\left\lceil -\log\delta\right\rceil \log\left|\letters\right|$ random
bits.\end{thm}
\begin{proof}
Since the sketch of Theorem \ref{thm:main} overestimates, 
\[
\hat{\Phi}_{w}\left(\sigma\right)\geq\Phi_{w}\left(\sigma\right)\text{ for all }w\in\alph
\]
it follows that every $w\in\alph$ with $\Phi_{w}\left(\sigma\right)>\rho$
will be an element of $\hat{\mathcal{H}}$. On the other hand, by
Theorem \ref{thm:main} $\hat{\Phi}_{w}\left(\sigma\right)>\Phi_{w}\left(\sigma\right)+\epsilon\left\Vert \Phi\left(\sigma\right)\right\Vert _{1,\left(M\right)}$
occurs with probability smaller than $\delta$. Hence, if $\Phi_{w}\left(\sigma\right)<\rho-\epsilon\left\Vert \Phi\left(\sigma\right)\right\Vert _{1,\left(M\right)}$,
the probability of erroneously including $w$ is less than $\delta$.
The computational complexity is the same as in Theorem \ref{thm:main}
up to a constant since we only additionally check $\Phi^{h}\left(\sigma\right)>\rho$
at every event. \end{proof}
\begin{rem}
Similar to the count-min sketch for finding the set of heavy hitter
letters, choosing the threshold $\rho$ can be an issue: if $\rho$
is too big, the set of heavy hitters stays empty, if it is choose
too small it the set $\mathcal{H}_{M}\left(\sigma\right)$ becomes
quickly too large. One approach is to run above heavy hitter pattern
sketch simultaneous for different thresholds and stop calculating
the associated heavy hitter list as soon as it becomes too big. Note
that this can be done with the same sketch $\hat{\Phi}$, thus the
additional computational overhead is minor.
\end{rem}

\section{\label{sec:Experiments}Experiments}

The sketch of Theorem \ref{thm:main} provides an estimate for every
coordinates $\Phi_{w}\left(\sigma\right)$, not only the heavy hitter
coordinates. In general it is an interesting question how to summarize
the quality of a sketch, that is to find a ``loss function''. Since
the $\ell_{1}$ norm naturally appears in the estimates, we record
in the experiments the difference between $\Phi\left(\sigma\right)$
and $\hat{\Phi}_{w}\left(\sigma\right)$, i.e. 
\begin{equation}
\operatorname{Error}_{M}:=\frac{1}{M}\sum_{m=1}^{M}\operatorname{error}_{m}\text{, where }\operatorname{error}_{m}:=\frac{\sum_{\left|w\right|=m}\left|\Phi_{w}\left(\sigma\right)-\hat{\Phi}_{w}\left(\sigma\right)\right|}{\left\Vert \Phi\left(\sigma\right)\right\Vert _{1,\left(M\right)}}\label{eq:l1 loss}
\end{equation}
denotes the relative error on the level of the $m$-th ordered moments.
Recall that the sketch always overestimates, $\hat{\Phi}_{w}\left(\sigma\right)\geq\Phi_{w}\left(\sigma\right)$,
hence $\left\Vert \hat{\Phi}\left(\sigma\right)\right\Vert _{1,\left(M\right)}\geq\left\Vert \Phi\left(\sigma\right)\right\Vert _{1,\left(M\right)}$.

\subsection{Experiment 1. Stream sketch}

We apply Theorem \ref{thm:main} to a fixed stream $\sigma=\left(\sigma_{i}\right)_{i=1}^{L}$
consisting of $L=10^{6}$ events that are drawn from an alphabet with
$\left|\letters\right|=100$ letters. This is a toy example in the
sense that $\left|\letters\right|=100$ is not sufficient for many
real-world examples. However, we find it instructive since it allows
to calculate and store the ground truth, i.e.~$\Phi\left(\sigma\right)$
(6 hours on a multicore machine for words up to length $3$) to compare
it with the quality of the sketch. Table \ref{tab:The-stream-} and
summarizes the performance of the sketch from Theorem \ref{thm:main}.

\begin{table}[H]
\begin{centering}
\begin{tabular}{|c|c|c|c|c|}
\hline 
$\left|\alphsmall\right|$ & Nr.~of hashes & Events/second & $\frac{\text{memory for }\Phi\left(\sigma\right)}{\text{memory for }\hat{\Phi}\left(\sigma\right)}$  & $\operatorname{Error}_{3}$\tabularnewline
\hline 
4 & 2 & 64368 & 6012.50 & 3923.47 \tabularnewline
\hline 
4 & 4 & 33529.3 & 3006.25 & 3048.61\tabularnewline
\hline 
4 & 8 & 17651.8 & 1503.13 & 2927.01\tabularnewline
\hline 
4 & 16 & 9120.63 & 751.56 & 2086.38\tabularnewline
\hline 
4 & 32 & 4620.79 & 375.78 & 2061.50\tabularnewline
\hline 
8 & 2 & 13231.9 & 864.81 & 484.71 \tabularnewline
\hline 
8 & 4 & 6688.18 & 432.41 & 364.40\tabularnewline
\hline 
8 & 8 & 3411.47 & 216.20 & 293.34\tabularnewline
\hline 
8 & 16 & 1712.27 & 108 & 268.00\tabularnewline
\hline 
8 & 32 & 850.85  & 54.05 & 230.30\tabularnewline
\hline 
16 & 2 & 1567.09 & 115.63 & 57.52\tabularnewline
\hline 
16 & 4 & 781.82 & 57.81 & 42.94\tabularnewline
\hline 
16 & 8 & 390.48 & 28.91 & 38.66\tabularnewline
\hline 
16 & 16 & 194.98 & 14.45 & 33.14\tabularnewline
\hline 
16 & 32 & 97.213 & 7.23 & 26.29\tabularnewline
\hline 
32 & 2 & 775.96 & 1.87 & 8.15\tabularnewline
\hline 
32 & 4 & 388.06 & 7.47 & 6.59\tabularnewline
\hline 
32 & 8 & 195.25 & 3.73 & 5.01\tabularnewline
\hline 
32 & 16 & 97.93 & 1.87 & 4.41\tabularnewline
\hline 
32 & 32 & 49.21 & 0.99 & 3.60\tabularnewline
\hline 
\end{tabular}
\par\end{centering}

\caption{\label{tab:The-stream-}The stream $\sigma=\left(1,a_{i}\right)_{i=1}^{10^{5}}$
with $a_{i}\in\protect\letters$, $\left|\protect\letters\right|=100$.
There are $10$ letters that make about $10$ percent of the events,
the rest of the events is uniformly distributed among the remaining
$90$ letters. Calculating the ground-truth $\left(\Phi_{w}\left(\sigma\right)\right)_{\left|w\right|\leq3}$
took 6 hours on a modern multicore computer (Intel Xeon, CPU E5-2690
v4, 2.60GHz, 56 CPUs with 2 threads per core). Without sketching,
one needs to update for every of the $10^{5}$ events $10^{2}+10^{4}+10^{6}$
real numbers. }
\end{table}

\subsection{Experiment 2: classifying Markov chains}

We sampled two types of random streams of constant counter increase
$\lambda_{i}\equiv1$ from an alphabet consisting of$\left|\letters\right|=10^{5}$
letters. Each stream contains $L=10^{5}$ events and two heavy hitters
which we denote below wlog as $\left\{ 1,2\right\} $. The first type
of stream was sampled as follows: in the first $2500$ events we choose
at each step with probability $p$ the letter $1$ and with probability
$1-p$ uniformly from $\letters\backslash\left\{ 1,2\right\} $; in
the following $2500$ events we do the same with $2$ instead of ;
in the last $5000$ events we choose with probability $p$ uniformly
from $\left\{ 1,2\right\} $ and with $1-p$ uniformly form $\letters\backslash\left\{ 1,2\right\} $.
The second type of streams was constructed analogous but the heavy
hitters were run in reverse order (i.e.~$2$ as heavy hitter in first
$2500$ events occurring with probability $p$, followed by $1$ as
heavy hitter in next $2500$ events occurring with probability $q$,
uniformly from both in the last $5000$ events). Note that these are
Markov chains (adding time as state space to account for the change
of regime). It is trivial for standard sketches to identify the heavy
hitters $\left\{ 1,2\right\} $, but as $p$ approaches $q$, it becomes
harder to distinguish the streams. On the other hand, taking order
information into account allows for perfect classification, see Table
\ref{tab:For--and} below.

\begin{table}[H]
\begin{centering}
\begin{tabular}{|c|c|c|c|}
\hline 
$q-p$ & $q$ & Mean accuracy $M=1$  & Mean accuracy $M=2$\tabularnewline
\hline 
$0.001$ & $0.101$ & 0.57 & 1.0\tabularnewline
\hline 
$0.005$ & $0.105$ & 0.63 & 1.0\tabularnewline
\hline 
$0.01$ & $0.11$ & 0.785 & 1.0\tabularnewline
\hline 
$0.02$ & $0.12$ & 0.89 & 1.0\tabularnewline
\hline 
$0.03$ & $0.13$ & 0.975 & 1.0\tabularnewline
\hline 
\end{tabular}
\par\end{centering}

\caption{\label{tab:For--and}For $p=0.1$ and each of above values for $q$
we sampled $500$ streams for each type. We then ran heavy hitter
pattern sketch (using $20$ hash functions to an alphabet of $10$
letters) for each stream and computed the corresponding features (i.e.~$\left\{ \hat{\Phi}_{w}\left(\sigma\right):w\in\left\{ 1,2,11,12,21,22\right\} \right\} $
for $M=2$ resp.~$\left\{ \hat{\Phi}_{w}\left(\sigma\right):w\in\left\{ 1,2\right\} \right\} $
for $M=1$) and trained a logistic regression classifier ($\ell_{1}$
penalty via $5$-fold cross-validation over the training set, training-test
splitted as $0.8$ to $0.2$). Above shows the mean accuracy on the
testing set. As $\left|p-q\right|$ becomes small, $M=1$ approaches
the uninformed baseline of guessing the type ($0.5$ mean accuracy
is achieved by choosing uniformly at random the type of the stream).
In contrast, $M=2$ uses second order information which allows for
perfect classification. Calculating the heavy hitter sketch took approximately
$0.15$ seconds per stream.}
\end{table}

\begin{acknowledgement}
This work was supported by ERC advanced grant ESig (no.~291244).

\newpage{}
\end{acknowledgement}
\bibliographystyle{plain}
\bibliography{/scratch/Dropbox/projects/Latex/roughpaths}

\newpage{}

\appendix

\section{\label{sec:top stream}A topology on streams and universality}

We want to regard streams that have similar length and similar events
as close to each other. It is therefore natural to view them as paths
in the vector space spanned by letters in $\letters$. To make this
precise, we use the Donsker embedding 
\[
\iota:\mathcal{S}\hookrightarrow C\left(\left[0,1\right],\mathbb{R}^{\left|\letters\right|}\right)
\]
by identifying $\sigma$ as a continuous path $\gamma_{\sigma}\in C\left(\left[0,1\right],\mathbb{R}^{\left|\letters\right|}\right)$
by using $\left(a_{1},\ldots,a_{\left|\letters\right|}\right)$ as
a orthonormal basis of $\mathbb{R}^{\left|\letters\right|}$ and setting
\begin{eqnarray}
\gamma_{\sigma}\left(t\right) & := & \left(tL-\left\lfloor tL\right\rfloor \right)\lambda_{\left\lceil tL\right\rceil }a_{\left\lceil tL\right\rceil }+\sum_{i=1}^{\left\lfloor tL\right\rfloor }\lambda_{i}a_{i}\text{ for }j=1,\ldots,L.\label{eq:stream2path}
\end{eqnarray}
In words: the path $\gamma_{\sigma}$ is a lattice path starts at
$t=0$ at the origin in $\mathbb{R}^{\left|\letters\right|}$ and
upon receiving an event $\left(\lambda,a\right)$ goes with constant
speed in direction $a$ for a time proportional to $\lambda$. Denote
with 
\[
C^{1-var}\left(\left[0,1\right],\mathbb{R}^{\left|\letters\right|}\right)=\left\{ \gamma\in C\left(\left[0,1\right],\mathbb{R}^{\left|\letters\right|}\right):\left\Vert \gamma\right\Vert _{1}=\sup_{0\leq t_{1}<\cdots<t_{M}\leq1}\sum_{i=1}^{M-1}\left|\gamma\left(t_{i+1}\right)-\gamma\left(t_{i}\right)\right|<\infty\right\} 
\]
the set of bounded variation paths.
\begin{defn}
Denote with $\iota:\mathcal{S}\rightarrow C^{1-var}\left(\left[0,1\right],\mathbb{R}^{\left|\letters\right|}\right)$
the Donsker embedding. We equip $\mathcal{S}$ with the pullback topology
of $\iota$, that is the open sets in $\mathcal{S}$ are $\iota^{-1}\left(U\right)$
where $U$ is an open set in $C^{1-var}\left(\left[0,1\right],\mathbb{R}^{\left|\letters\right|}\right)$.\end{defn}
\begin{prop}
\label{prop:universal}Let the event map $p:\mathcal{E}\rightarrow\mathcal{F}$
and resulting feature map $\Phi:\mathcal{S}\rightarrow\mathcal{F}$
be as in Theorem \ref{thm:sig hopf algebra}. For every $f\in C\left(\mathcal{S},\mathbb{R}\right)$
$\epsilon>0$, and compact set $K\subset\mathcal{S}$ there exists
a $\ell\in\ta$ such that 
\[
\sup_{\sigma\in K}\left|f\left(\sigma\right)-\left\langle \ell,\Phi\left(\sigma\right)\right\rangle \right|<\epsilon.
\]
\end{prop}
\begin{proof}
By Stone--Weierstrass we need to verify that $\left\{ \sigma\mapsto\left\langle \ell,\Phi\left(\sigma\right)\right\rangle ,\ell\in\ta\right\} $
is a point-separating subalgebra of $C\left(\mathcal{S},\mathbb{R}\right)$.
The subalgebra property follows since 
\[
\left\langle \ell,\Phi\left(\sigma\right)\right\rangle \left\langle \ell^{\prime},\Phi\left(\sigma\right)\right\rangle =\left\langle m\left(\ell\otimes\ell^{\prime}\right),\Phi\left(\sigma\right)\right\rangle 
\]
where $m$ denotes the shuffle (resp.~infiltration product) on $\ta$
as detailed in Appendix \ref{sec:Hopf}. The fact that $\sigma\mapsto\Phi\left(\sigma\right)$
is point-separating follows in the case of the event map $p\left(\lambda,a\right)=\exp\lambda a$
from classic results (injectivity of the signature for non-treelike
paths); in the case $p\left(\lambda,a\right)=1+\lambda a$ note that
if $\sigma=\left(\sigma_{i}\right)_{i=1}^{L}$ and $\tau=\left(\tau_{i}\right)_{i=1}^{K}$
and $L<K$ then $\Phi_{w}\left(\sigma\right)=0\neq\Phi_{w}\left(\tau\right)>0$
since $\Phi_{w}\left(\tau\right)$ is a sum over disjoint time indices.
If $L=K$, and $\sigma_{i}=\left(\lambda_{i},a_{i}\right)$, $\tau_{i}=\left(\rho_{i},b_{i}\right)$
and there exists a $i$ such that $a_{i}\neq b_{i}$ then the result
follows immediately by comparing at the coordinate $w=a_{1}\cdots a_{L}$.
Finally, if $L=K$ and $a_{i}=b_{i}$ for all $i$, we can argue as
in \cite{lyons2017hyperbolic}.
\end{proof}

\section{\label{sec:Hopf}Some (Hopf) algebras}

We describe the interplay between feature and dual space using Hopf
algebras. This is a concise way to capture the interplay between $\mathcal{F}$
and its dual space; a reader less familiar with algebra might want
to skip this appendix after a brief look at Theorem \ref{thm:magnus features}
and Theorem \ref{thm:scaling}.

\subsection{Hopf algebras}

Hopf algebras arise naturally when a linear space $\mathcal{H}$ as
well as its dual $\mathcal{H}^{\star}$ are equipped with products
$m$ and $m^{\star}$ such that $\left(\mathcal{H},m\right)$ and
$\left(\mathcal{H}^{\star},m^{\star}\right)$ are algebras and $m$,$m^{\star}$
are ``compatible''. This is of interest to us, since the feature space
$\tb$ has the concatenation product (``a new event happens'') and
we will that its dual $\ta$ can be equipped with commutative product.
This elegantly describes the interplay between dual and feature space
and allows us to address Points (\ref{enu:linearizes}), (\ref{enu:features algebra})
(universality of features and structure preserving) from Section \ref{sub:Features}
\begin{example}
Consider a finite dimensional, linear space $\mathcal{H}$ and let
$\left(\mathcal{H},m\right)$ and $\left(\mathcal{H}^{\star},m^{\star}\right)$
be algebras. The product $m$ can be written as a linear map 
\[
m:\mathcal{H}\otimes\mathcal{H}\rightarrow\mathcal{H}
\]
that fulfills associativity and distributivity. By duality, $m$ can
be encoded as a linear map $\Delta_{m}$ on the dual $\mathcal{H}^{\star}$
\begin{eqnarray}
\Delta_{m}:\mathcal{H}^{\star} & \rightarrow & \mathcal{H}^{\star}\otimes\mathcal{H}^{\star}\label{eq:dual coproduct}\\
\text{ by }\left\langle \Delta^{\star}\left(\ell\right),g\otimes h\right\rangle  & := & \left\langle \ell,m\left(g\otimes h\right)\right\rangle \text{ for }f,g\in\mathcal{H},\ell\in\mathcal{H}^{\star}.\nonumber 
\end{eqnarray}
Thus instead of working with two algebras $\left(\mathcal{H},m\right)$
and $\left(\mathcal{H}^{\star},m^{\star}\right)$, we can work with
one space $\left(\mathcal{H}^{\star},m^{\star},\Delta_{m}\right)$
and two linear maps, 
\begin{eqnarray*}
m^{\star}: & \mathcal{H}^{\star}\otimes\mathcal{H}^{\star} & \rightarrow\mathcal{H}^{\star}\\
\Delta_{m}: & \mathcal{H}^{\star} & \rightarrow\mathcal{H}^{\star}\otimes\mathcal{H}^{\star}
\end{eqnarray*}
{[}or vice versa with $\left(\mathcal{H},m,\Delta_{m^{\star}}\right)$,
$\Delta_{m^{\star}}$ defined analogous to (\ref{eq:dual coproduct}){]}.
A natural way to ensure ``compatibility'' of the two algebras is to
require that $\Delta_{m}$ is an algebra morphism, 
\[
m^{\star}\left(\Delta_{m}\left(h\right)\otimes\Delta_{m}\left(g\right)\right)=\Delta\left(m^{\star}\left(h\otimes g\right)\right)\text{ for }g,h\in\mathcal{H}^{\star}.
\]
Further, an algebra also has unit $\epsilon$ which can be represented
as linear map $\mathcal{\epsilon}:\mathbb{R}\rightarrow\mathcal{H}$.
Again by duality this unit $\epsilon$ translates to a linear map
$\mathcal{H}^{\star}\rightarrow\mathbb{R}$ called \emph{counit.}
 More general, a not necessarily finite dimensional vector space
$\mathcal{H}$ equipped with two linear maps, called \emph{product
}and \emph{coproduct}, 
\[
m:\mathcal{H}\otimes\mathcal{H}\rightarrow\mathcal{H}\text{ and }\Delta:\mathcal{H}\rightarrow\mathcal{H}\otimes\mathcal{H}
\]
that fulfill natural generalizations of above properties is called
a \emph{bialgebra}. If additionally (as in our application) the space
$\mathcal{H}$ is connected and graded, it is a non-trivial result
that there must exist a so-called \emph{antipode}\footnote{ Takeuchi's formula reads $S=\sum_{k\geq0}\left(-1\right)^{k}m^{\left(k-1\right)}\left(id-\eta\epsilon\right)^{\otimes k}\Delta^{\left(k-1\right)}$
where $\eta$,$\epsilon$ are unit and counit of $\mathcal{H}$.}\emph{ 
\[
S:\mathcal{H}\rightarrow\mathcal{H}
\]
}and we call $\left(\mathcal{H},m,\Delta,S\right)$ a \emph{Hopf algebra}
(we refer to the monographs \cite{cartier2007primer,MR0252485,MR540015}
for details). This antipode has an intuitive interpretation: it is
simply the inverse map of a group structure inside the linear space
$\mathcal{H}$, that is if we denote with 
\[
G\left(\mathcal{H}\right)=\left\{ h\in\mathcal{H}:\epsilon\left(h\right)=1,\,\Delta\left(h\right)=h\otimes h\right\} 
\]
the set of \emph{group-like }elements\emph{, }then $\left(G\left(\mathcal{H}\right),m\right)$
is a group with inverse given by $S\left(g\right)=g^{-1}$. In many
situations, the elements of $\mathcal{H}$ have a combinatorial interpretation
in which case one may think of $m$ and $\Delta$ as \emph{composition}
and \emph{decomposition}. Elements that are ``simple under decomposition''
with $\Delta$, are exactly $\mathcal{G}\left(\mathcal{H}\right)$.
We use below that linear functionals of group-like elements are closed
under multiplication, that is if $h\in G\left(\mathcal{H}\right)\subset\mathcal{H}$
is group-like, then 
\[
\left\langle \ell,h\right\rangle \left\langle \ell^{\prime},h\right\rangle =\left\langle \ell\otimes\ell^{\prime},h\otimes h\right\rangle =\left\langle \ell\otimes\ell^{\prime},\Delta\left(h\right)\right\rangle =\left\langle m^{\star}\left(\ell\otimes\ell^{\prime}\right),h\right\rangle \text{for }\ell,\ell^{\prime}\in\mathcal{H}^{\star}.
\]

\end{example}

\subsection{Feature and dual space for streams}

By construction of $\Phi\left(\sigma\right)$, the obvious choice
for multiplication in $\tb$ --- no matter what event map $p:\mathcal{E}\rightarrow\tb$
is used --- is the non-commutative multiplication. That is we define
the so-called \emph{concatenation product} $\concat$

\[
\concat:\tb\otimes\tb\rightarrow\tb,\,\,\,\,\concat\left(a_{1}\cdots a_{M}\otimes a_{M+1}\cdots a_{M+N}\right):=a_{1}\cdots a_{M+N}\text{ for }a_{1},\ldots,a_{m+n}\in\letter
\]
extended by linearity to $\tb$. This turns $\left(\tb,\concat\right)$
into a \emph{non-commutative algebra.} To learn about $\sigma$, we
apply linear functionals $\ell\in\ta$ to $\Phi_{p}\left(\sigma\right)\in\tb$.
To find a product $\shuffle$ that turns the dual $\ta$ into an algebra,
it is useful to recall that our features should be group-like. Hence,
this product 
\[
\shuffle:\ta\otimes\ta\rightarrow\ta
\]
depends highly on the choice of event map $p$. The existence of this
product is a priori not clear; a non-trivial results is that for both
choices $p\left(\lambda,a\right)=1+\lambda a$ and $p\left(\lambda,a\right)=\exp\left(\lambda a\right)$
such a product $m_{p}$ exists. This turns $\left(\ta,m_{p}\right)$
into a \emph{commutative algebra.} The final step is to capture this
structure of two two algebras $\left(\tb,\concat\right)$ and $\left(\ta,\shuffle\right)$
by using one Hopf algebra. 
\begin{thm}
\label{thm:sig hopf algebra}\label{thm:magnus features}Define the
event map $p:\mathcal{E}\rightarrow\mathcal{F}$ as either 
\[
\left(\lambda,a\right)\mapsto\exp\left(\lambda a\right)\text{ or }\left(\lambda,a\right)\mapsto1+\lambda a,
\]
and define a feature map $\Phi:\mathcal{S}\rightarrow\mathcal{F}$
as 
\[
\Phi\left(\sigma\right)=\prod_{i=1}^{L}p\left(\sigma_{i}\right)\text{ for streams }\sigma=\left(\sigma_{i}\right)\text{ of events }\sigma_{i}=\left(\lambda_{i},a_{i}\right).
\]
Then
\begin{enumerate}
\item \label{enu:moments}The features $\Phi\left(\sigma\right)\in\mathcal{F}$
are ordered moments, that is 
\begin{equation}
\Phi\left(\sigma\right)=\mathbb{E}\left[\lambda_{i_{1}}\cdots\lambda_{i_{m}}a_{i_{1}}\cdots a_{i_{m}}\right],\label{eq:moments}
\end{equation}
with expectation taken over the order statistics of $\left(i_{1},\ldots,i_{m}\right)$
sampled uniformly without replacement from $\left\{ 1,\ldots,L\right\} $
in the case $p\left(\lambda,a\right)=1+\lambda a$; in the case $p\left(\lambda,a\right)=\exp\left(\lambda a\right)$
the expectation is taken over the order statistic of $\left(i_{1},\ldots,i_{L}\right)$
sampled uniformly with replacement from $\left\{ 1,\ldots,L\right\} $.
\item \label{enu:coordinates}The coordinates $\left\langle \Phi\left(\sigma\right),w\right\rangle $
are given as
\[
\left\langle \Phi\left(\sigma\right),w\right\rangle =\begin{cases}
\sum_{\left(1\right)}\frac{\lambda\left(\boldsymbol{i}\right)}{\boldsymbol{i}!} & \text{if }p\left(\lambda,a\right)=\exp\left(\lambda a\right),\\
\sum_{\left(2\right)}\lambda\left(\boldsymbol{i}\right) & \text{ if }p\left(\lambda,a\right)=1+\lambda a.
\end{cases}
\]
where the sum $\sum_{\left(1\right)}$ is taken over all $\boldsymbol{i}=\left(i_{1},\ldots,i_{L}\right)\in\mathbb{N}^{L}$
such that $i_{1}\leq\cdots\leq i_{L}$ and $a_{i_{1}}\cdots a_{i_{L}}=w$;
for the sum $\sum_{\left(2\right)}$ we additionally assume $i_{1}<\cdots<i_{L}$.
We denote $\lambda\left(\boldsymbol{i}\right)=\lambda_{i_{1}}\cdots\lambda_{i_{M}}\in\mathbb{R}$
and $\boldsymbol{i}!$ is recursively defined as $\boldsymbol{i}!=i_{1}!$
if $L=1$ and $\left(i_{1},\ldots,i_{L},i_{L+1}\right)!=\left(i_{1},\ldots,i_{L}\right)!k!$
where $k=\max_{j\geq0:i_{L+1-j}=i_{L+1}}j$.
\item \label{enu:magnus is hopf}There exists a linear map $\shuffle:\ta\otimes\ta\rightarrow\ta$
such that $\left(\ta,\shuffle,\coconcat\right)$ is a commutative
Hopf algebra. 
\end{enumerate}
\end{thm}

\begin{proof}
Points (\ref{enu:moments}) and (\ref{enu:coordinates}) follow by
direct calculations. For $p\left(\lambda,a\right)=\exp\left(\lambda a,\right)$,
Point (\ref{enu:magnus is hopf}) is a standard result in algebra
and the resulting Hopf algebra is known as the shuffle Hopf algebra,
see e.g.~\cite{reutenauer-93}. The case $p\left(\lambda,a\right)=1+\lambda a$
seem to be much less known. The key is to define the commutative product
on $\ta$ recursively as 
\[
\infil\left(av\otimes bw\right)=a\inf\left(v\otimes bw\right)+b\inf\left(av\otimes b\right)+avw1_{a=b}
\]
Note the additional third term, which is $0$ in the shuffle Hopf
algebra case. Then one can directly verify the properties of a Hopf
algebra in a (lengthy) direct calculation. This commutative product
is known as the infiltration product, see \cite{lothaire1997combinatorics}.\end{proof}
\begin{rem}
~
\begin{enumerate}
\item Above theorem is classic for $p\left(\lambda,a\right)=\exp\left(\lambda a\right)$,
the product $\shuffle$ on \emph{$\mathcal{\ta}$ }is the ``shuffle
product''. If $p\left(\lambda,a\right)=1+\lambda a$, the product
is the less standard ``infiltration product''.
\item In the special case of constant counterincrease, $\lambda_{i}\equiv1$,
the coordinates count how often a substring $w$ appears in $\sigma$
and the choice of the event map $p$ determines if an event can be
counted several times. The choice $p\left(\lambda,a\right)=1+\lambda a$
recovers the string kernel features from \cite{lodhi2002text,2016arXiv160108169K}.
The Hopf algebra structure for string kernels features seems to be
new.
\item If we replace the non-commutative product $\concat$ in $\mathcal{F}$
with a commutative product, the features reduce to usual (unordered)
sample moments as seen by (\ref{eq:moments}).
\item The event map $\left(\lambda,a\right)\mapsto1+\lambda a$ and associated
Hopf algebra generalizes classic string kernel features. While the
coordinates are more intuitive, the resulting feature map is much
less robust under scaling than $\left(\lambda,a\right)\mapsto\exp\left(\lambda a\right)$,
see the section below.
\item We stick to the machine learning terminology and call $\mathcal{F}$
feature space and $\ta$ dual space. From an algebraic perspective
it is the other way around: $\mathcal{F}=\mathbb{R}\langle\langle\letters\rangle\rangle$
is the algebraic dual of $\mathcal{\ta}=\mathbb{R}\langle\letters\rangle$. 
\item A technical point: $\left(\mathcal{\ta},m_{p},\coconcat\right)$ is
a Hopf algebra, so we expect a Hopf algebra on feature space $\mathcal{F}$
by duality. However, one needs to work with a slightly smaller subset
of $\mathcal{F}$ (called the Sweedler dual).
\end{enumerate}
\end{rem}
\begin{example}
As remarked in above proof, the commutative product on $\ta$ known
recursively. For example~
\[
m_{p}\left(ab\otimes ba\right)=\begin{cases}
abba+2abba+2baab+baba & \text{if }p\left(\lambda,a\right)=\exp\left(\lambda a\right),\\
aba+bab+abab+2abba+2baab+baba & \text{if }p\left(\lambda,a\right)=1+\lambda a.
\end{cases}
\]

\end{example}

\subsection{Scaling limits and estimates}

One can identify $\sigma=\left(\lambda_{i},a_{i}\right)_{i=1}^{L}$
as an element of $\mathbb{R}^{\left|\letters\right|^{L}}$ and apply
standard features $\varphi$ for vector-valued data. Most reasonable
choice of $\varphi$ guarantee that 
\[
f\left(v\right)\approx\left\langle \ell,\varphi\left(v\right)\right\rangle \text{ for }v\in\mathbb{R}^{\left|\letters\right|^{L}}
\]
However, this approach becomes infeasible for large $L$ and additionally
methodological problems arise if our data contains streams of different
length. On the other hand, we will see that 
\[
f\left(\sigma\right)\approx\left\langle \ell,\Phi\left(\sigma\right)\right\rangle 
\]
and that above expression still makes sense in the scaling limit $L\rightarrow\infty$. 
\begin{thm}
\label{thm:scaling}Let the event map $p:\mathcal{E}\rightarrow\mathcal{F}$
and resulting feature map $\Phi:\mathcal{S}\rightarrow\mathcal{F}$
be as in Theorem \ref{thm:sig hopf algebra}. 
\begin{enumerate}
\item \label{enu:magnus grouplike}Let $\left(\sigma^{n}\right)_{n\geq1}\subset\mathcal{S}$
be a sequence of streams such that $\gamma_{\sigma^{n}}$ converges
in $\left(C^{1-var}\left(\left[0,1\right],\mathbb{R}^{\left|\letters\right|}\right),\left\Vert \cdot\right\Vert _{1}\right)$
to a path $\gamma=\left(\gamma^{a}\right)_{a\in\letters}$. Then for
every word $w=a_{1}\cdots a_{M}$, 
\[
\lim_{n}\left\langle \Phi\left(\sigma^{n}\right),w\right\rangle =\int_{0\leq t_{1}\leq\cdots\leq t_{m}\leq1}d\gamma^{a_{1}}\left(t_{1}\right)\cdots d\gamma^{a_{m}}\left(t_{m}\right)
\]
and $\lim_{n}\Phi\left(\sigma^{n}\right)$ is a group-like element
of the Hopf algebra $\mathcal{F}$.
\item \label{enu: magnus norm}Define $\left\Vert \sigma\right\Vert _{1}=\sum_{i=1}^{L}\left|\lambda_{i}\right|$.
Then $\left\Vert \sigma\right\Vert _{1}=\left\Vert \gamma_{\sigma}\right\Vert _{1}$
and 
\[
\left\Vert \Phi\left(\sigma\right)\right\Vert _{1,M}\leq\frac{\left\Vert \sigma\right\Vert _{1}^{M}}{M!}\text{ and }\left\Vert \Phi\left(\sigma\right)\right\Vert _{1}\leq\exp\left(\left\Vert \sigma\right\Vert _{1}\right)\text{ for }\sigma\in\mathcal{S}.
\]
 with equality if $p\left(\lambda,a\right)=\exp\left(\lambda,a\right)$. 
\item \label{enu:Stream-concatenation-is}Stream concatenation is multiplication
in feature space, 
\[
\Phi\left(\sigma,\tau\right)=\Phi\left(\sigma\right)\cdot\Phi\left(\tau\right)\text{ for }\sigma,\tau\in\mathcal{S}.
\]

\end{enumerate}
\end{thm}
\begin{proof}
Point (\ref{enu:magnus grouplike}) follows by a direct calculation
for both event maps. For Point (\ref{enu: magnus norm}) note that
under the event map $p\left(\lambda,a\right)=\exp\left(\lambda a\right)$
we have that $\left\langle a_{1}\cdots a_{M},\Phi\left(\sigma\right)\right\rangle =\int_{0<t_{1}<\cdots<t_{M}<L}d\gamma_{\sigma}^{a_{1}}\left(t_{1}\right)\cdots d\gamma_{\sigma}^{a_{M}}\left(t_{M}\right)$
where $\gamma_{\sigma}$ denotes the path (\ref{eq:stream2path}).
Then 
\begin{eqnarray*}
\left\langle a_{1}\cdots a_{M},\Phi\left(\sigma\right)\right\rangle  & = & \int_{0<t_{1}<\cdots<t_{M}<L}d\gamma_{\sigma}^{a_{1}}\left(t_{1}\right)\cdots d\gamma_{\sigma}^{a_{M}}\left(t_{M}\right)\\
 & = & \frac{1}{M!}\int_{\left[0,L\right]^{M}}d\gamma_{\sigma}^{a_{1}}\left(t_{1}\right)\cdots d\gamma_{\sigma}^{a_{M}}\left(t_{M}\right)\\
 & = & \frac{1}{M!}\prod_{m=1}^{M}\int_{\left[0,L\right]}d\gamma_{\sigma}^{a_{m}}\left(t\right)=\frac{1}{M!}\prod_{m=1}^{M}f_{a_{m}}.
\end{eqnarray*}
where denoted with $f_{a}$ the frequency of letter $a$ in $\sigma$,
i.e.~$f_{a}=\sum_{i:a_{i}=a}\lambda_{i}$. Hence, 
\[
\left(\sum_{w:\left|w\right|=M}\left\langle w,\Phi\left(\sigma\right)\right\rangle \right)=\sum_{a_{1},\ldots,a_{M}\in\letters}f_{a_{1}}\cdots f_{a_{M}}=\left(\sum_{a\in\letters}f_{a}\right)^{M}=\left(\sum_{i=1}^{L}\lambda_{i}\right)^{M}=\left\Vert \sigma\right\Vert ^{M}.
\]
For the event map $p\left(\lambda,a\right)=1+\lambda a$, Theorem
\ref{thm:sig hopf algebra} shows that the feature coordinates are
smaller equal than then ones for the event map $p\left(\lambda,a\right)=\exp\left(\lambda,a\right)$.
Hence, the result follows. Point (\ref{enu:Stream-concatenation-is})
follows directly from our construction of features as $\Phi\left(\sigma\right)=\prod_{i=1}^{L}p\left(\lambda_{i},a_{i}\right)$
for $\sigma=\left(\lambda_{i},a_{i}\right)$.
\end{proof}
\pagebreak{}

\section{Algorithms\label{sec:Algorithms}}

\begin{algorithm}
\begin{algorithmic}[h!] 
\init
\State{Set $h\leftarrow{-\lceil \log \delta \rceil}$ and $d\leftarrow \lceil \frac{1}{\epsilon}\rceil$} 
\State{Sample $H_1,\ldots,H_h$ hash functions from a 2-universal hash family}
\State{For each $H_i$ initialize $d+d^2+\cdots+d^M$ counters to store $\Phi_i=\Phi(H_i(\sigma))$}
\Statex
\Require
\State{Fetch event $(\lambda,a)$ in stream}
\While{$(\lambda,a)\neq \square$} \Comment{$\square$ denotes the last element/end of the stream}
\For{$i=1,\ldots,h$} \State{$\Phi_i\leftarrow \Phi_i*p(\lambda,H_i(a))$} \Comment{power series multiplication}

\EndFor
\State{Fetch next event $(\lambda,a)$} 
\EndWhile 
\Statex
\Ensure{On query $w$, return $\min_{i\in \{1,\ldots,h\}} \langle \Phi_i,H_i(w) \rangle$}
\end{algorithmic}

\caption{\label{alg:Signature-sketch}$\left(\epsilon,\delta,M\right)$-sketch.}
\end{algorithm}

\begin{algorithm}
\begin{algorithmic}[h!] 
\init
\State{Prepare $(\epsilon,\delta)$ sketch structure $\hat{\Phi}$ as in Algorithm \ref{alg:Signature-sketch}}
\State{$\hat{\mathcal{H}}\leftarrow \{1\}$}\Comment{$1$ denotes the empty word}
\Statex
\Require
\State{Fetch event $(\lambda,a)$ in stream}
\While{$(\lambda,a)\neq \square$} \Comment{$\square$ denotes the last element/end of the stream}
\For{$i=1,\ldots,h$} \State{$\Phi_i\leftarrow \Phi_i*p(\lambda,H_i(a))$} \Comment{power series multiplication}
\EndFor
\If{$\min_{i\in \{1,\ldots,h\}} \langle \Phi_i,H_i(a) \rangle>\rho$}
\State{$\hat{\mathcal{H}}\leftarrow\hat{\mathcal{H}}\cup \{a\}$}
\EndIf
\State{Fetch next event $(\lambda,a)$} 
\EndWhile
\For{$w\in \hat{\mathcal{H}}^\star$}
\If{$\min_{i\in \{1,\ldots,h\}} \langle \Phi_i,H_i(w) \rangle\geq\rho^{M}$}
\State{$\hat{\mathcal{H}}_M\leftarrow\hat{\mathcal{H}}_M\cup \{w\}$}
\EndIf
\EndFor
\Statex
\Ensure{Return $\hat{\mathcal{H}}_M$}
\end{algorithmic}

\caption{\label{alg:Signature-sketch-1-2}Heavy hitter patterns}
\end{algorithm}

\end{document}